\documentclass[runningheads]{llncs}

\usepackage{arxiv}
\usepackage{subcaption}

\usepackage{graphicx}
\usepackage{color}
\usepackage{amsmath}
\usepackage{amssymb}
\usepackage{booktabs}
\usepackage{algorithm}
\usepackage{algpseudocode}
\newtheorem{hyp}{Hypothesis}
\usepackage{bm}
\usepackage{sidecap}
\usepackage{caption}

\usepackage[accsupp]{axessibility}  %

\usepackage[pagebackref,breaklinks,colorlinks]{hyperref}

\usepackage{orcidlink}

\usepackage[capitalize]{cleveref}

\crefname{section}{Sec.}{Secs.}
\Crefname{section}{Section}{Sections}
\Crefname{table}{Table}{Tables}
\crefname{table}{Tab.}{Tabs.}

\usepackage{authblk}

\begin{document}

\setlength{\abovedisplayskip}{4pt} %
\setlength{\belowdisplayskip}{4pt} %
\setlength\textfloatsep{0pt}
\setlength\floatsep{2mm}
\setlength\textfloatsep{2mm}
\setlength\intextsep{2mm}
\setlength\abovecaptionskip{2mm}

\title{Flooding Regularization for Stable Training of \\ Generative Adversarial Networks}

\author[1]{Iu Yahiro}
\author[2,1]{Takashi Ishida}
\author[1,2]{Naoto Yokoya}
\affil[1]{The University of Tokyo, Bunkyo-ku, Tokyo, Japan}
\affil[2]{RIKEN AIP, Chuo-ku, Tokyo, Japan}

\maketitle
\begin{abstract}
Generative Adversarial Networks (GANs) have shown remarkable performance in image generation.
However, GAN training suffers from the problem of instability.
One of the main approaches to address this problem is to modify the loss function, often using regularization terms in addition to changing the type of adversarial losses. 
This paper focuses on directly regularizing the adversarial loss function.
We propose a method that applies flooding, an overfitting suppression method in supervised learning, to GANs to directly prevent the discriminator's loss from becoming excessively low.
Flooding requires tuning the flood level, but when applied to GANs, we propose that the appropriate range of flood level settings is determined by the adversarial loss function, supported by theoretical analysis of GANs using the binary cross entropy loss.
We experimentally verify that flooding stabilizes GAN training and can be combined with other stabilization techniques.
We also show that by restricting the discriminator's loss to be no less than the flood level, the training proceeds stably even when the flood level is somewhat high.

\end{abstract}

\section{Introduction}
Generative Adversarial Networks (GANs) are one of the learning frameworks for generative models proposed by Goodfellow et al.~\cite{NIPS2014_5ca3e9b1}, and they have shown remarkable performance in a wide range of image generation tasks~\cite{park2019SPADE,CycleGAN2017,8099502,Demir2018PatchBasedII,choi2020stargan}.
GANs are based on a training strategy where two models, a generator $G$ and a discriminator $D$, are trained adversarily.
The generator takes a noise vector $z$ sampled from a known distribution $\mathbb{P}_z$ (usually the standard normal distribution) as input and produces generated data $G(\bm{z})$ as output.
The discriminator $D$ takes either real data sampled from the target underlying distribution $\mathbb{P}_r$ or generated data as input $\bm{x}$ and outputs the probability $D(\bm{x})$ that the input is real data.
The discriminator aims to correctly distinguish between real and generated data, while the generator aims to reduce the discriminator's performance on generated data.
By designing the loss function in this way, Jensen-Shannon divergence of generated data distribution $\mathbb{P}_g$ and $\mathbb{P}_r$ is minimized~\cite{NIPS2014_5ca3e9b1}. 
Goodfellow et al.~\cite{NIPS2014_5ca3e9b1} defined this adversarial structure using min-max formulation of the following value function $V$, involving the generator $G$ and the discriminator $D$:

\begin{equation}
\begin{split}
\label{eq:minmax}
    \underset{G}{\min}\ \underset{D}{\max}\ V(D,G) =&
    \mathbb{E}_{\bm{x}\sim{\mathbb{P}_{r}}}[\log(D(\bm{x}))] + \mathbb{E}_{\bm{z}\sim{\mathbb{P}_{z}}}[\log (1-D(G(\bm{z}))].
\end{split}
\end{equation}

\begin{figure}[t]
 \begin{minipage}[b]{0.325\linewidth}
  \centering
  \includegraphics[keepaspectratio, width=0.8\linewidth]{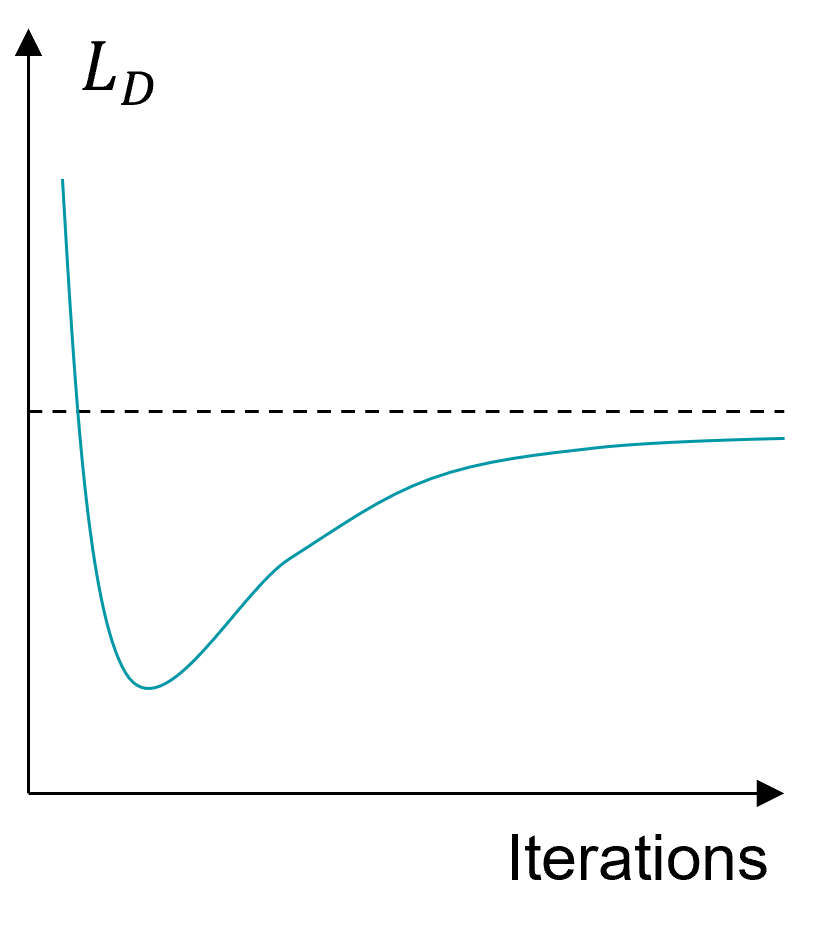}
  \subcaption{Ideal training}
 \end{minipage}
 \begin{minipage}[b]{0.325\linewidth}
  \centering
  \includegraphics[keepaspectratio, width=0.8\linewidth]{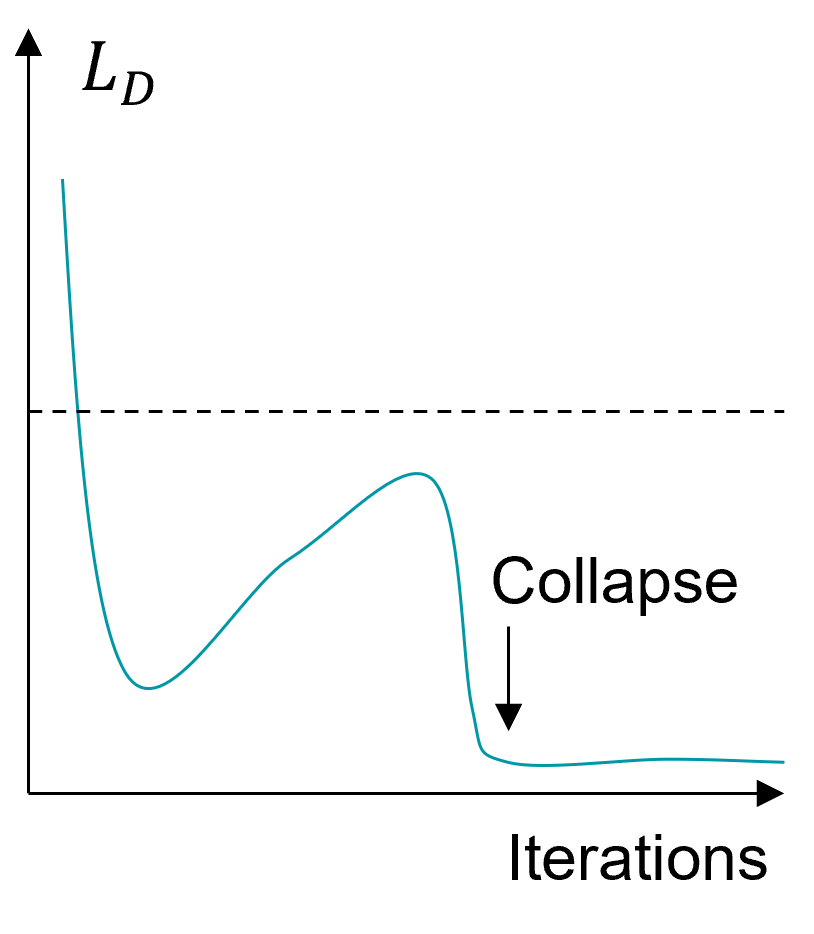}
  \subcaption{Training collapse}
 \end{minipage} 
 \begin{minipage}[b]{0.325\linewidth}
  \centering
  \includegraphics[keepaspectratio, width=0.87\linewidth]{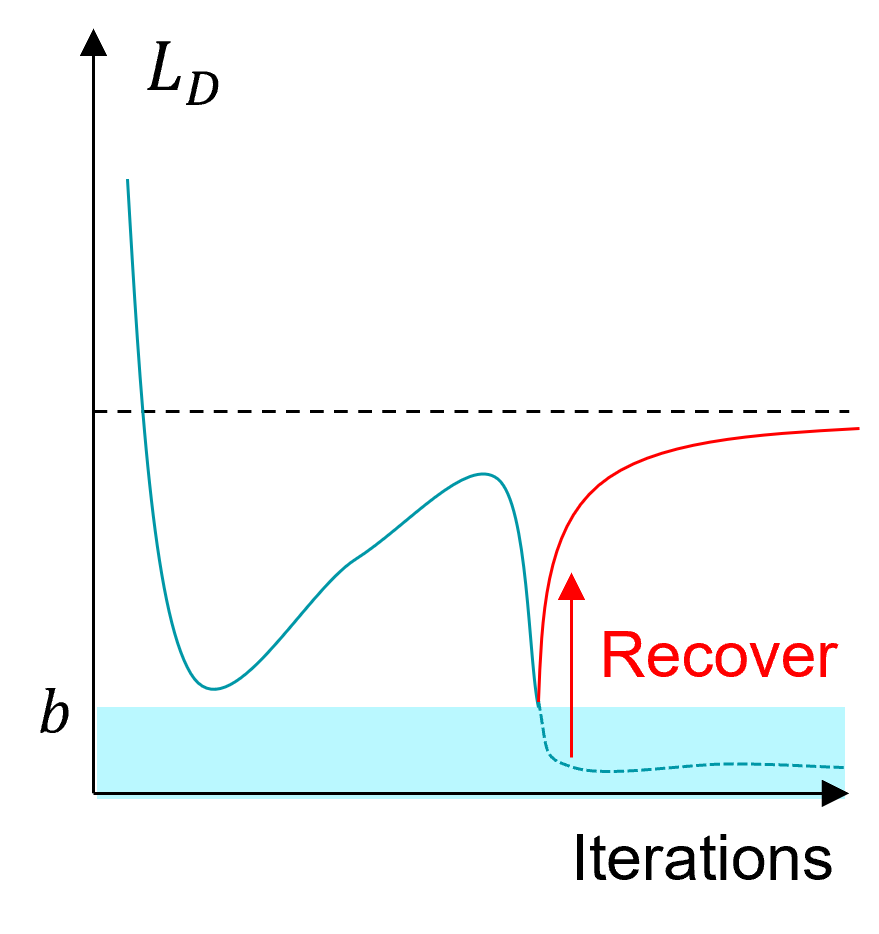}
  \subcaption{Effect of flooding}
 \end{minipage}
 \caption{
 Illustration of the discriminator's loss ($L_{D}$) progression in GANs and the effect of flooding.
 In ideal training (a), the discriminator's loss rises steadily to loss at the convergence (dotted line), and the training will converge.
 On the other hand, when the training collapses (b), the discriminator's loss falls sharply and then stays at a very low value.
 By applying flooding with flood level $b$, we can suppress the rapid decline in the discriminator's loss, thereby stabilizing the training.
 }
 \label{flooding_effect_concept}
\end{figure}

Although GANs successfully tackle image generation with this approach, GAN training suffers from instability. 
Previous work~\cite{arjovsky2017towards,NIPS2016_8a3363ab} pointed out that the training is unstable because of the instability of the discriminator, the discriminator's loss becomes excessively low, and the discriminator overwhelms the generator.
In training, the generator is updated based on the discriminator's predictions.
However, Arjovsky et al.~\cite{arjovsky2017towards} showed theoretically that if the discriminator is always optimal, it leads to a vanishing gradient of the generator and unstable training.

Previous research has proposed methods to solve the instability.
These can be categorized as changes in adversarial losses, regularization, and architectural changes~\cite{10.1145/3439723}.
The adversarial loss is the loss that creates an adversarial structure between the generator and discriminator.
For example, GANs originally used the binary cross entropy loss (BCE loss) as the adversarial loss based on the min-max formula \eqref{eq:minmax}. 
However, previous research~\cite{lim2017geometric,Mao_2017_ICCV,pmlr-v70-arjovsky17a} showed that using the BCE loss causes instability and proposed replacement of the BCE loss.

Another main approach is the addition of a regularization term to the adversarial loss that leads to training stabilization.
It can be combined with changes in adversarial losses without affecting the theoretical convergence ($\mathbb{P}_g = \mathbb{P}_r$).
For example, gradient penalty~\cite{NIPS2017_892c3b1c} adds a regularization term that keeps the gradient norm close to 1, which prevents gradient explosion and vanishing.
WGAN-GP~\cite{NIPS2017_892c3b1c} uses it to enforce Lipschitz constraint.
The loss functions of the discriminator and generator can be divided into the adversarial loss and the addition of regularization terms as
\begin{equation}
\begin{split}
\label{eq:GANs_loss_concept}
    L_{D} &= L_{D,\mathrm{adv}} + \sum_i \lambda_{D,i} L_{D, \mathrm{aux}, i} , \ 
    L_{G} = L_{G,\mathrm{adv}} + \sum_i \lambda_{G,i} L_{G, \mathrm{aux}, i} ,
\end{split}
\end{equation}
where $L_{D,\mathrm{adv}}$ and $L_{G,\mathrm{adv}}$ are the adversarial losses for the discriminator and generator, respectively, and $L_{D,\mathrm{aux},i}$ and $L_{G,\mathrm{aux},i}$ are the $i$th regularization terms with the weight coefficients $\lambda_{D,i}$ and $\lambda_{G,i}$.
The adversarial losses are calculated using the discriminator outputs for real or generated data as
\begin{equation}
\begin{split}
\label{eq:GANs_adv_loss_concept}
    L_{D,\mathrm{adv}} &= L_{D,\text{real}} + L_{D,\text{fake}} 
    = \mathbb{E}_{\bm{x}\sim{\mathbb{P}_{r}}}[f_{D,\text{real}}(\bm{x})] + \mathbb{E}_{\bm{x}\sim{\mathbb{P}_{g}}}[f_{D,\text{fake}}(\bm{x})], 
    L_{G,\mathrm{adv}} = \mathbb{E}_{\bm{x}\sim{\mathbb{P}_{g}}}[f_{G}(\bm{x})],
\end{split}
\end{equation}
where $L_{D,\text{real}}$ and $L_{D,\text{fake}}$ are the discriminator's losses for real and generated data, respectively, and $f_{D,\text{real}}$, $f_{D,\text{fake}}$, and $f_{G}$ are functions from the discriminator's outputs to the losses.
Some studies further improved the stability by combining changes in the type of adversarial losses with the addition of regularization terms.
For example, WGAN-GP~\cite{NIPS2017_892c3b1c} improves performance with gradient penalty as a regularization term and Wasserstein loss~\cite{pmlr-v70-arjovsky17a} as the adversarial loss.
However, adding regularization terms requires tuning $\lambda_{D,aux,i}$ and $\lambda_{G,aux,i}$.

This paper proposes a direct regularization technique for the adversarial loss values to stabilize GAN training.
We explore a new technique that directly prevents the discriminator from becoming too accurate and taking too low loss.
This low-loss state can be regarded as the discriminator's overfitting to the current distribution.
Therefore, we propose the application of \textit{flooding}~\cite{ishida2021need}, a method for preventing overfitting in supervised image classification, to GANs.
In order to prevent an excessive decrease in classification loss $L$ when the prediction model overfits, flooding recalculates loss $h$ given by
\begin{equation}
\begin{split}
\label{eq:flooding_ishida}
    h(L,b)=|L-b|+b .
\end{split}
\end{equation}
Here, $b$ is called the flood level. Due to the absolute value, the gradient will be flipped when the loss becomes smaller than $b$, preventing the overfitting.
Note that the adding back $b$ of Eq.~\eqref{eq:flooding_ishida} does not affect the gradient but ensures that $h(L,b)=L$ where $L\geq b$~\cite{ishida2021need}.
The flood level is a hyperparameter that requires tuning.
Our contributions are as follows.

\begin{enumerate}
\setlength{\parskip}{0cm} %
\setlength{\itemsep}{0cm} %
\item We propose to apply flooding, a simple method that prevents overfitting in supervised learning, to GAN training. It stabilizes the training process, as depicted in Figure \ref{flooding_effect_concept}.
\item Unlike in supervised learning, the discriminator's losses at the training convergence are uniquely determined, as proven theoretically. We introduce a novel approach to set the flood level based on the losses at the training convergence.
\item We demonstrate that flooding stabilizes GAN training experimentally. We also show that flooding is effective when the flood level is not too low and in combination with existing stabilization methods, such as changes in adversarial loss type and architecture. Furthermore, we demonstrate that applying flooding only to either $L_{D,\text{real}}$ or $L_{D,\text{fake}}$ significantly impacts performance, indicating whether the discriminator overfits real or generated data.
\end{enumerate}

\section{Related Work}
We first review stabilization methods for GANs. Next, we show applications of adversarial architectures and methods for overfitting in supervised learning.

\subsection{Stabilization methods for training GANs}
\label{sec:related_work_1}
There are three categories in stabilization methods for the GAN training: change of adversarial loss type, regularization, and change of architectures~\cite{10.1145/3439723}.

\begin{table}[t]
  \begin{center}
    {\small{
\begin{tabular}{llll}
\toprule
Adversarial loss  & $f_{D,\text{real}}$& $f_{D,\text{fake}}$ & $f_{G}$ \\
\midrule
BCE loss~\cite{NIPS2014_5ca3e9b1} & $- \log (D(\bm{x})) $ & $- \log (1-D(G(\bm{z})))$& $ \log (1-D(G(\bm{z})) $ \\
BCE loss (non-saturating)~\cite{NIPS2014_5ca3e9b1} & $- \log (D(\bm{x})) $ & $- \log (1-D(G(\bm{z})))$& $ - \log (D(G(\bm{z})) $ \\
Wasserstein loss~\cite{pmlr-v70-arjovsky17a} & $-D(\bm{x})$ & $D(\bm{x})$ & $-D(G(\bm{z}))$ \\
Hinge loss~\cite{lim2017geometric} & $\max (0, 1-D(\bm{x}))$ & $\max (0, 1+D(G(\bm{z}))$ & $-D(G(\bm{z}))$ \\
Least squares loss~\cite{Mao_2017_ICCV} & $\frac{1}{2} (D(\bm{x})-b_{\mathrm{LS}})^2$ & $\frac{1}{2} (D(G(\bm{z}))-a_{\mathrm{LS}})^2$ & $\frac{1}{2} (D(G(\bm{z}))-c_{\mathrm{LS}})^2$ \\
\bottomrule
\end{tabular}
}}
\end{center}
\caption{
Examples of adversarial loss functions. One of the setting $a_{\mathrm{LS}}$, $b_{\mathrm{LS}}$, and $c_{\mathrm{LS}}$ proposed by Mao et al.~\cite{Mao_2017_ICCV} is $(0,1,1)$.
}
\label{adv_loss_table}
\end{table}

\noindent\textbf{Change of adversarial loss type}
Goodfellow et al.~\cite{NIPS2014_5ca3e9b1} pointed out that the BCE loss can lead to a problem that the discriminator can correctly discriminate the generated data ($D(G(\bm{z}))\simeq 0$), resulting in saturated gradients.
To address it, a non-saturating BCE loss is proposed by modifying $f_G$ shown in Table~\ref{adv_loss_table}.

However, the instability still remains~\cite{arjovsky2017towards,NIPS2016_8a3363ab}.
Goodfellow et al.~\cite{NIPS2014_5ca3e9b1}have shown that the training with the BCE loss minimizes the Jensen-Shannon divergence between the real data distribution $\mathbb{P}_r$ and generated data distribution $\mathbb{P}_g$.
Arjovsky et al.~\cite{pmlr-v70-arjovsky17a} showed that the minimization causes instability, and they proposed Wasserstein loss based on Earth Mover (EM) distance.
Moreover, Lim et al.~\cite{lim2017geometric} proposed a hinge loss, and Mao et al.~\cite{Mao_2017_ICCV} proposed a least squares loss based on minimizing Peason $\chi^2$ distance. 
Table~\ref{adv_loss_table} shows the losses.
These methods can be regarded as integral probability metric (IPM)-based regularization~\cite{10.1145/3439723}, where the generators and discriminators belong to a particular function class, such as models with Lipschitz continuity.

It is crucial to mathematically prove that $\mathbb{P}_g$ at training convergence matches  $\mathbb{P}_r$ when changing the adversarial loss functions~\cite{NIPS2014_5ca3e9b1,lim2017geometric,Mao_2017_ICCV}. 
Note that the proof assumes infinite data from $\mathbb{P}_r$ and ideal models for the training.
Therefore, it cannot be perfectly reproduced in experiments, but it is useful to make theoretical analysis easy and acquire insightful knowledge.
Previous research~\cite{NIPS2014_5ca3e9b1,Mao_2017_ICCV} follows the proof procedure described below.
\begin{enumerate} 
\setlength{\parskip}{0cm} %
\setlength{\itemsep}{0cm} %
\label{GANs_convergence_steps}
    \item Find a discriminator $D^*$, which minimizes the loss for a fixed generator $G$.
    \item Find a generator $G_{\mathrm{opt}}$, which minimizes $D^*$ loss.
\end{enumerate}
Let $L_{D^*}=L_{D^*,real}+L_{D^*,fake}$ denote the loss of $D^*$, $L_{D_{\mathrm{opt}}}$ denote discriminator's loss in step 2 ($D_{\mathrm{opt}}$), and $L_{G_{\mathrm{opt}}}$ denote generator's loss in step 2.
\begin{table}[t]
  \begin{center}
    {\small{
\begin{tabular}{lllll}
\toprule
Adversarial loss & $L_{D_{\mathrm{opt}}, \mathrm{real}}$ & $L_{D_{\mathrm{opt}}, \mathrm{fake}}$ & $L_{D_{\mathrm{opt}}}$ & $L_{G_{\mathrm{opt}}}$ \\
\midrule
BCE loss~\cite{NIPS2014_5ca3e9b1} & $\log 2$ & $\log 2$& - & $-\log 2$ \\
BCE loss (non-saturating)~\cite{NIPS2014_5ca3e9b1} & $\log 2$ & $\log 2$& - & $\log 2$ \\
Wasserstein loss~\cite{pmlr-v70-arjovsky17a} & - & - & $0$ & - \\
Hinge loss~\cite{lim2017geometric} & - & - & $2$ & $1$ \\
Least squares loss~\cite{Mao_2017_ICCV} & $\frac{(a_{\mathrm{LS}}-b_{\mathrm{LS}})^2}{8}$ & $\frac{(a_{\mathrm{LS}}-b_{\mathrm{LS}})^2}{8}$ & - & $\frac{(a_{\mathrm{LS}}+b_{\mathrm{LS}}-2c_{\mathrm{LS}})^2}{8}$ \\
\bottomrule
\end{tabular}
}}
\end{center}
\caption{
The theoretical loss values for each adversarial loss when the training reaches the convergence. 
If $L_{D_{\mathrm{opt}}, \mathrm{real}}$ and $L_{D_{\mathrm{opt}}, \mathrm{fake}}$ cannot be uniquely determined, only $L_{D_{\mathrm{opt}}}$ is shown, and if those can be determined, the apparent $L_{D_{\mathrm{opt}}}$ ($ = L_{D_{\mathrm{opt}}, \mathrm{real}}+L_{D_{\mathrm{opt}}, \mathrm{fake}}$) is omitted.
Also, omit the Wasserstein loss $L_{G_{\mathrm{opt}}}$ since it is not uniquely determined.
}
\label{adv_loss_opt_table}
\end{table}
Table \ref{adv_loss_opt_table} shows $L_{D_{\mathrm{opt}}, \mathrm{real}}$, $L_{D_{\mathrm{opt}}, \mathrm{fake}}$, $L_{D_{\mathrm{opt}}}$, and $L_{G_{\mathrm{opt}}}$ for each adversarial loss. 

\noindent\textbf{Regularization}
There are various regularization techniques to stabilize the GAN training.
One major approach is the addition of regularization terms to the adversarial loss.
This directly prevents gradient explosion and vanishing~\cite{DBLP:journals/corr/KodaliAHK17,petzka2018on}, model overfitting~\cite{brock2018large,pmlr-v70-arjovsky17a}, and mode collapse~\cite{che2017mode}.
Gradient penalty~\cite{NIPS2017_892c3b1c} is an example of adding a regularization term, which improves the discriminator stability by adding a squared error between the gradient norm and $1$ so that the gradient norm approaches $1$.
It contributes to Lipschitz continuity and training stabilization.
These methods require tuning of the coefficients to balance the adversarial loss.
Label smoothing~\cite{NIPS2016_8a3363ab} regularizes through the target labels.
However, its stabilization effect on adversarial losses other than the BCE loss is unknown.
Normalization is another common approach to regularization~\cite{xiang2017effects,miyato2018spectral}. Spectral normalization~\cite{miyato2018spectral}, a representative normalization technique for GANs, stabilizes training with Lipschitz continuity of the discriminator through the normalization of the weight matrices.
Unlike the existing regularization techniques for GANs, our method directly regularizes the adversarial loss.

\noindent\textbf{Change of architecture}
Changing the architecture is a commonly used approach to stabilization.
For example, when generating images with a high resolution, a method that efficiently preserves the entire image features is essential. 
Deep convolutional GAN (DCGAN)~\cite{radford2016unsupervised}, which employs a convolutional layer, and self-atteintion GAN (SAGAN)~\cite{pmlr-v97-zhang19d}, which introduces an attention mechanism, have been proposed.
Some studies~\cite{karras2018progressive,brock2018large,Karras_2019_CVPR,karras2020analyzing} have proposed to generate high-resolution and photo-realistic images by devising architectures. 
These model changes are relatively easy to combine with loss changes because they do not disrupt the competing structure, which is represented by adversarial losses.

\subsection{Various application of adversarial architectures}
Adversarial architectures for an efficient high-dimension distribution generator training have been applied in many research fields.
For instance, Mirza et al.~\cite{mirza2014conditional} proposed conditional GANs to control generated images with labels.
The idea was employed for a wide range of applications, such as text-to-image generation~\cite{stackgan} and image-to-image translation~\cite{pix2pix}.
Additionally, domain-adversarial neural network (DANN)~\cite{dann}, adversarial discriminative domain adaptation (ADDA)~\cite{Tzeng_2017_CVPR}, and Wasserstein distance guided representation learning (WDGRL)~\cite{wdgrl} proposed the use of adversarial frameworks for domain adaptation, which alleviates the gap between source and target domain in classification tasks.
These architectures also have challenges with training multiple models.

\subsection{Methods for overfitting in supervised learning}
In supervised learning, overfitting is a well-known phenomenon in which a model performs well on training data but fails on unknown data.
Typical methods for preventing overfitting in supervised learning are dropout~\cite {JMLR:v15:srivastava14a}, batch normalization~\cite{pmlr-v37-ioffe15}, data augmentation~\cite{Shorten2019}, and label smoothing~\cite{rethinking}.
Many of these techniques can be applied not only to supervised learning but also to other domains, and they are sometimes adopted in GANs as well~\cite{10.1145/3569928}.
Ishida et al.~\cite{ishida2021need} proposed flooding, a method that recalculates the loss based on the formula \eqref{eq:flooding_ishida} with respect to the flood level $b$ so that the loss does not become extremely small and avoids overfitting.
On the other hand, Xie et al.~\cite{xie2022iflood} proposed individual flood (iFlood) to apply flooding before taking the expected value of losses.
The method is effective because it can regularize only the loss of overfitted instances.
While these methods perform well in supervised learning, there is no standard for setting $b$, and it requires a tuning process.

\section{Method}
\label{sec:Method}
In this section, we first provide an overview of GAN training and propose the application of flooding to GANs. In the proposal, we show that there are several ways to apply flooding to GANs and discuss the flood level setting.

\subsection{Overview of GAN training}
GAN training has a generator $G$ and a discriminator $D$ and proceeds based on losses defined in Eq.~\eqref{eq:GANs_loss_concept}. 
This section assumes simple GANs without any regularization term to simplify the discussion.
For example, with the BCE loss~\cite{NIPS2014_5ca3e9b1}, $L_D$ and $L_G$ can be written as 
\begin{equation}
\begin{split}
L_{D} 
=&\mathbb{E}_{\bm{x}\sim \mathbb{P}_r}[- \log (D(\bm{x}))] + \mathbb{E}_{\bm{x} \sim \mathbb{P}_g}[- \log (1-D(\bm{x}))] , 
L_{G} =\mathbb{E}_{\bm{x}\sim \mathbb{P}_g}[\log (1-D(\bm{x}))] .
\end{split}
\end{equation}

\subsection{Application of flooding to GANs}
Instability in GANs can occur when the discriminator's loss is too low, indicating overfitting of the discriminator.
We aim to improve the stability by applying flooding, which is a method for preventing overfitting in supervised learning.
In this case, we need to consider ``how to apply flooding to GANs'' and ``how to set the flood level,'' which are explained in the following sections.

\noindent\textbf{How to apply flooding to GANs}
\label{sec:method_adopt_flooding}
We propose to apply flooding to $L_{D}$ to avoid the overfitting of the discriminator.
There are three ways to apply flooding, depending on the inserting position of the operation $h$ defined in Eq.~\eqref{eq:flooding_ishida},
\begin{equation}
\begin{split}
\label{eq:GANs_adv_loss_flood}
L_{D,\text{flood},1} 
=&\mathbb{E}_{\bm{x}\sim \mathbb{P}_r}[h(f_{D,\text{real}}(D(\bm{x})), b_{\,\text{real}})] + \mathbb{E}_{\bm{x} \sim \mathbb{P}_g}[h(f_{D,\text{fake}}(D(\bm{x})), b_{\,\text{fake}})] ,
\\
L_{D,\text{flood},2} 
=&h(\mathbb{E}_{\bm{x}\sim \mathbb{P}_r}[f_{D,\text{real}}(D(\bm{x}))], b_{\,\text{real}}) + h(\mathbb{E}_{\bm{x} \sim \mathbb{P}_g}[f_{D,\text{fake}}(D(\bm{x}))], b_{\,\text{fake}}) ,
\\
L_{D,\text{flood},3} 
=& h(\mathbb{E}_{\bm{x}\sim \mathbb{P}_r}[f_{D,\text{real}}(D(\bm{x}))] + \mathbb{E}_{\bm{x} \sim \mathbb{P}_g}[f_{D,\text{fake}}(D(\bm{x}))], b_{\,\text{all}}), 
\end{split}
\end{equation}
where $b_{\,\text{real}}$ and $b_{\,\text{fake}}$ are the flood levels for adversarial losses for real and generated data, respectively, and $b_{\,\text{all}}$ is the flood level for the sum of the adversarial losses.
As flooding and iFlood, the difference in the flood level and inserting position can cause performance improvement.

\noindent\textbf{How to set the flood level}
\label{sec:flood_level}
The appropriate setting of the flood level $b_{\,\text{real}}$, $b_{\,\text{fake}}$, and $b_{\,\text{all}}$ is important.
In supervised learning, because the loss at convergence is not uniquely determined but it depends on models and datasets,
the flood level is a hyperparameter, and the appropriate range of it has not yet been shown.
On the other hand, we propose the following hypotheses about the setting of the flood level for GANs by using property that $L_{D_{\mathrm{opt}}, \mathrm{real}}$, $L_{D_{\mathrm{opt}}, \mathrm{fake}}$, or the sum of the losses $L_{D_{\mathrm{opt}}}$ is uniquely determined, as summarized in Table \ref{adv_loss_opt_table}. 

\begin{hyp} 
\label{hyp:standard_1}
If $L_{D_{\mathrm{opt}}, \mathrm{real}}$ and $L_{D_{\mathrm{opt}}, \mathrm{fake}}$ are uniquely determined, then $b_{\,\mathrm{real}}$ and $b_{\,\mathrm{fake}}$ should be set to satisfy two conditions, $b_{\,\mathrm{real}} < L_{D_{\mathrm{opt}}, \mathrm{real}} $ and $b_{\,\mathrm{fake}} < L_{D_{\mathrm{opt}}, \mathrm{fake}}$, for $L_{D,\mathrm{flood},1} $ or $L_{D,\mathrm{flood},2}$.
\end{hyp}
\begin{hyp} 
\label{hyp:standard_2}
If $L_{D_{\mathrm{opt}}}$ is uniquely determined, then $b_{\,\mathrm{real}}$ and $b_{\,\mathrm{fake}}$ should be set to satisfy a condition $b_{\,\mathrm{real}} + b_{\,\mathrm{fake}} < L_{D_{\mathrm{opt}}}$, for $L_{D,\mathrm{flood},1}$ or $L_{D,\mathrm{flood},2}$. Moreover, $b_{\,\mathrm{all}}$ should be set to satisfy a condition $b_{\,\mathrm{all}} < L_{D_{\mathrm{opt}}}$, for $L_{D,\mathrm{flood},3}$.
\end{hyp}

If $L_{D_{\mathrm{opt}}, \mathrm{real}}$ and $L_{D_{\mathrm{opt}}, \mathrm{fake}}$ are uniquely determined, we can suppose excessively low loss for each losses.
This is the inspiration for Hypothesis \ref{hyp:standard_1}.
For example, since $L_{D_{\mathrm{opt}}, \mathrm{real}}$ and $L_{D_{\mathrm{opt}}, \mathrm{fake}}$ are $\log 2$ for GANs with the BCE loss, $b_{\,\text{real}}$ and $b_{\,\text{fake}}$ should be set lower than $\log 2$.
On the other hand, there are cases where $L_{D_{\mathrm{opt}}}$ rather than $L_{D_{\mathrm{opt}}, \mathrm{real}}$ and $L_{D_{\mathrm{opt}}, \mathrm{fake}}$ are uniquely determined as the hinge loss.
In such cases, because $L_{D_{\mathrm{opt}}}$ is expressed as $L_{D_{\mathrm{opt}}, \mathrm{real}} + L_{D_{\mathrm{opt}}, \mathrm{fake}}$, we should set $b_{\,\text{real}}$ and $b_{\,\text{fake}}$ according to the sum.
Moreover, the setting $b_{\,\text{all}}$ should be lower than $L_{D_{\mathrm{opt}}}$ as Hypothesis \ref{hyp:standard_1}.
This is the inspiration for Hypothesis \ref{hyp:standard_2}.

To provide theoretical support for Hypotheses \ref{hyp:standard_1} and \ref{hyp:standard_2}, let us consider the case of GANs with the BCE loss.
For each adversarial loss, $\mathbb{P}_g = \mathbb{P}_r$ at the training convergence is proved following the procedure outlined in Section \ref{sec:related_work_1}.
In the early stages of the training, the difference between $\mathbb{P}_g$ and $\mathbb{P}_r$ is more significant, so $L_{D^*}$ is lower than $L_{D_{\mathrm{opt}}}$ because the discriminator can solve the discrimination task well by its distribution difference.
For example, Goodfellow et al.~\cite{NIPS2014_5ca3e9b1} showed that with the BCE loss and a fixed generator $G$, 
\begin{equation}
\label{eq:GANs_BCE_fixed_G_opt}
\begin{split}
D^*(\bm{x}) = \frac{\mathbb{P}_r(\bm{x})}{\mathbb{P}_r(\bm{x})+\mathbb{P}_g(\bm{x})} .
\end{split}
\end{equation}
It can also be proven that $L_{D^*,\mathrm{real}}$ and $L_{D^*,\mathrm{fake}}$ are smaller than $L_{D_{\mathrm{opt}}, \mathrm{real}}$ and $L_{D_{\mathrm{opt}}, \mathrm{fake}}$.
We can now show the following theorem.
\begin{theorem}
\label{Theorem_1}
In training GANs with the BCE loss based on $L_{D,\textup{flood},1}$, on $Supp(\mathbb{P}_r(x)) \cup Supp(\mathbb{P}_g(x))$,
\begin{equation}
\label{Theorem_1_formula_1}
\begin{aligned}
    D^*(\bm{x}) \in \{ 1-e^{-b_{\,\mathrm{fake}}}, e^{-b_{\,\mathrm{real}}}\},
\end{aligned}
\end{equation}
with $e^{-b_{\,\mathrm{real}}} + e^{-b_{\,\mathrm{fake}}} \leq 1$. 
On the other hand, with $e^{-b_{\,\mathrm{real}}} + e^{-b_{\,\mathrm{fake}}} > 1$,
\begin{equation}
\label{Theorem_1_formula_2}
\begin{split}
  \begin{cases}
    D^*(\bm{x})= \displaystyle\frac{\mathbb{P}_r(\bm{x})}{\mathbb{P}_r(\bm{x})+\mathbb{P}_g(\bm{x})} &
    \text{if\: $\mathbb{P}_g(\bm{x})$ satisfies} \text{ an inequality \eqref{condition},} \\
    D^*(\bm{x}) \in \{ 1-e^{-b_{\,\mathrm{fake}}}, e^{-b_{\,\mathrm{real}}}\} & \text{otherwise,}
  \end{cases}
\end{split}
\end{equation}
where the inequality is 
\begin{equation}
\label{condition}
\begin{aligned}
    1-e^{-b_{\,\mathrm{fake}}} < \frac{\mathbb{P}_r(\bm{x})}{\mathbb{P}_r(\bm{x})+\mathbb{P}_g(\bm{x})} < e^{-b_{\,\mathrm{real}}}.
\end{aligned}
\end{equation}
\end{theorem}

\begin{proof}
See Supplementary Section \ref{sec:proof_of_theorem_1}.
\end{proof}

The theorem is important in the following three points.
\begin{enumerate}
\setlength{\parskip}{0cm} %
\setlength{\itemsep}{0cm} %
    \item When $e^{-b_{\,\text{real}}} + e^{-b_{\,\text{fake}}} < 1$, the output of $D^*$ becomes a constant, which does not relate to $\mathbb{P}_g$ and $\mathbb{P}_r$, and then training will collapse.
    \item When $e^{-b_{\,\text{real}}} + e^{-b_{\,\text{fake}}} > 1$ and the inequality ~\eqref{condition} is held, the output of $D^*$ is the same as Eq.~\eqref{eq:GANs_BCE_fixed_G_opt}.
    Moreover, if $b_{\,\text{real}} < \log 2$ and $b_{\,\text{fake}} < \log 2$ is satisfied, $e^{-b_{\,\text{real}}} + e^{-b_{\,\text{fake}}} > 1$ and the inequality \eqref{condition} are held where $\mathbb{P}_g = \mathbb{P}_r$. 
    Because the goal of GANs is that the generator satisfies  $\mathbb{P}_g = \mathbb{P}_r$, $b_{\,\text{real}} < \log 2$ and $b_{\,\text{fake}} < \log 2$ is required for GAN training convergence.
    \item When $e^{-b_{\,\text{real}}} + e^{-b_{\,\text{fake}}} > 1$, the discriminator with a higher flood level is more difficult to satisfy the inequality \eqref{condition}.
\end{enumerate}
The details are shown in Supplementary Section \ref{sec:analysis_of_theorem_1}. We assume that higher flood levels increase the dangers mentioned in the third point while lower flood levels diminish the stabilizing effect. 
Therefore, it is necessary to investigate the appropriate setting of the flood level through experiments.

\section{Experiment}
\label{sec:Experiment}
We show experimentally the appropriate application and the effect of flooding on GANs discussed in Section~\ref{sec:Method}.

\subsection{Implementation}
We briefly summarize the implementation below. 
The details are provided in Supplementary Section \ref{sec:details_of_Experiments}.

\noindent\textbf{Synthetic Dataset} 
To examine the effect of flooding for training GANs, we used the ring of 2D Gaussian dataset (2D Ring) as previous research~\cite{lim2017geometric,MDGANs,NIPS2017_892c3b1c,cf-mc-gans,Mao_2017_ICCV}. We did five runs and evaluated the variety and quality of generated samples with `modes' and `high quality (HQ),' proposed in~\cite{VEEGANs}.
For instance, for 2D Ring, it holds that modes $\leq 8$ and HQ $\leq 1$.
In order to confirm how much flooding prevents mode collapse, we consider higher modes as better performance. If modes are the same, we consider one with higher HQ as better performance.

\noindent\textbf{DCGAN} 
We used unconditional DCGAN~\cite{radford2016unsupervised} to evaluate the performance of image generation. 
We used CIFAR10, CIFAR100 at 32$\times$32, STL10 at 64$\times$64, CelebA at both 64$\times$64 and 128$\times$128.
We conducted each experiment five times and evaluated the generated images using Fréchet Inception Distance (FID)~\cite{NIPS2017_8a1d6947}.

\noindent\textbf{Large model}
We conducted experiments using StarGAN V2~\cite{choi2020stargan} to investigate the effect of flooding in the generation of larger images.
We followed the author's implementation and used CelebAHQ.
We conducted each experiment five times and evaluated the generated images with FID and LPIPS~\cite{zhang2018unreasonable}. 

\begin{table}[t]
  \begin{center}
    {\small{
\begin{tabular}{lccccccc}
\toprule
Flooding type & Eval & w/o flooding & Small & Medium & Near Opt & Opt & Over Opt \\
\midrule
1 & Modes & 4.8 (1.2) & 6.6 (0.8) & \textbf{7.8 (0.4)} & 7.0 (0.0) & 2.0 (0.6) & 0.0 (0.0) \\
                       & HQ & 0.90 (0.11) & 0.94 (0.03) & \textbf{0.87 (0.07)} & 0.90 (0.03) & 0.00 (0.00) & 0.00 (0.00) \\
2 & Modes & - & 4.0 (2.0) & 4.2 (1.3) & \textbf{7.0 (0.6)} & 0.2 (0.4) & 0.0 (0.0) \\
                       & HQ & - & 0.65 (0.35) & 0.75 (0.18) & \textbf{0.28 (0.18)} & 0.00 (0.00) & 0.00 (0.00) \\
3 & Modes & - & 4.0 (2.3) & 5.2 (1.9) & \textbf{7.2 (0.4)} & 2.8 (1.6) & 0.0 (0.0) \\
                       & HQ & - & 0.68 (0.35) & 0.91 (0.06) & \textbf{0.24 (0.19)} & 0.01 (0.01) & 0.00 (0.00) \\
\bottomrule
\end{tabular}
}}
\end{center}
\caption{
Average (standard deviation) of modes and HQ with the BCE loss, different flooding types, and flood levels.
}
\label{table:result_synthetic_BCE}
\end{table}

\subsection{How to apply flooding}
\label{sec:flooding_bce}

First, we investigated the stabilizing effect of GAN training by flooding with a synthetic dataset and determined which flooding type in Eq.~\eqref{eq:GANs_adv_loss_flood} is appropriate.
We used the non-saturating BCE loss, which will be referred to as the BCE loss in the latter part of the paper.
Note that the discussion in Section \ref{sec:method_adopt_flooding} does not use $f_G=\log(1-D(G(z)))$, and we can make the same arguments with the non-saturating BCE loss.
To compare a case without flooding and to find the appropriate flood level, we conducted experiments on five different settings (Small, Medium, Near Opt, Opt, Over opt) with condition $b_{\,\text{real}} =b_{\,\text{fake}}$.
The details of the setting are described in Supplementary Section \ref{sec:details_of_Experiments}.

Table \ref{table:result_synthetic_BCE} shows the average and standard deviation of the modes and HQ.
Results without flooding indicate a high HQ (0.90) but low modes (4.8). This suggests that only a few of the eight Gaussian centers are accurately represented, implying that mode collapse occurred.
With flooding type $L_{D,\text{flood},1}$, the results of flood levels at Small, Medium, and Near Opt show that the modes are better and HQ achieves a high value. 
Specifically, with flood level Medium, the generator expresses all of the Gaussian centers in four out of five runs. 
On the other hand, with flooding type $L_{D,\text{flood},2}$ and $L_{D,\text{flood},3}$, the results of flood level Small and Medium are poor in mode, and the result of flood level Near Opt is poor in HQ.
This suggests that instance-level flooding $L_{D,\text{flood},1}$ and flood level under $L_{D_{\mathrm{opt}}, \mathrm{real}}$ and $L_{D_{\mathrm{opt}}, \mathrm{fake}}$ is effective in GAN training.
Because without flooding the discriminator avoids taking losses below the flood level, it also shows that training becomes unstable when the loss is too low at the instance-level.
It is also worth noting that the flood level Near Opt with flooding type $L_{D,\text{flood},1}$ stabilizes the training. 
It suggests that GAN training can progress even if the discriminator loses its ability to take low losses.
As shown in Theorem \ref{Theorem_1}, setting a larger flood level has a drawback that reduces the probability of satisfying the inequality \eqref{condition}. 
However, the experimental result suggests that preventing destabilization is more beneficial than the drawback.
We can also see that the training completely collapses when the flood level $b_{\,\text{real}}$, $b_{\,\text{fake}}$, and $b_{\,\text{all}}$ takes more than $L_{D_{\mathrm{opt}}, \mathrm{real}}$, $L_{D_{\mathrm{opt}}, \mathrm{fake}}$, and $L_{D_{\mathrm{opt}}}$.
As shown in Figure~\ref{fig:synthetic-mode-hq-transition} (a) and (b), compared with the modes and HQ transition without flooding, the collapse with flood level Opt and Over Opt started early on.
It indicates that such a configuration disrupts the GAN training, as discussed in Section \ref{sec:flood_level}.
We provide analysis on the loss and gradient with flooding in Supplementary Section \ref{sec:analysis_loss_gradient}.

\begin{figure}[t]
\begin{center}
\begin{minipage}{0.39\linewidth}
 \begin{minipage}[b]{\linewidth}
  \centering
  \includegraphics[keepaspectratio, width=\linewidth]{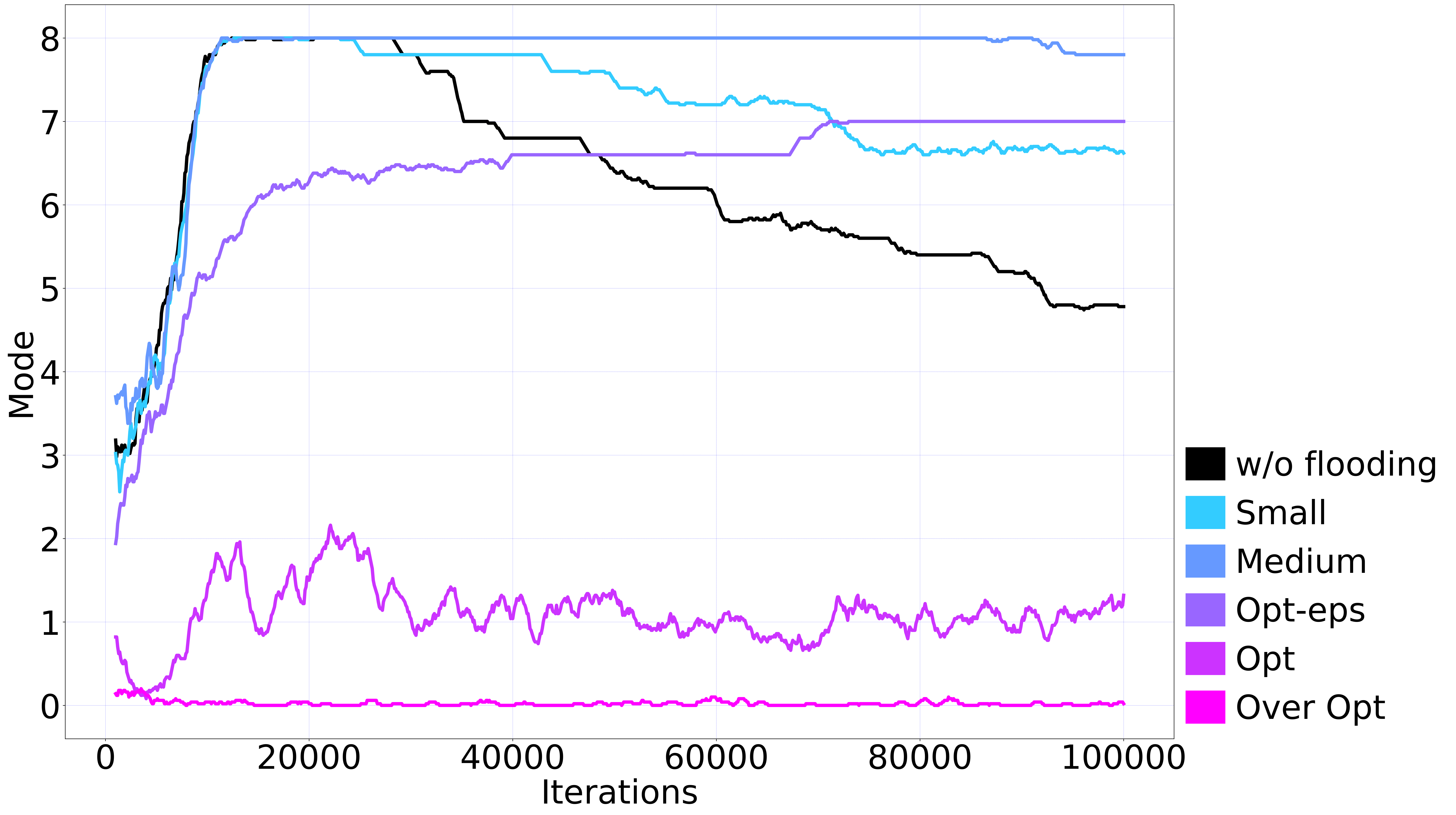}
  \subcaption{Transition of modes}
 \end{minipage} \\
 \begin{minipage}[b]{\linewidth}
  \centering
  \includegraphics[keepaspectratio, width=\linewidth]{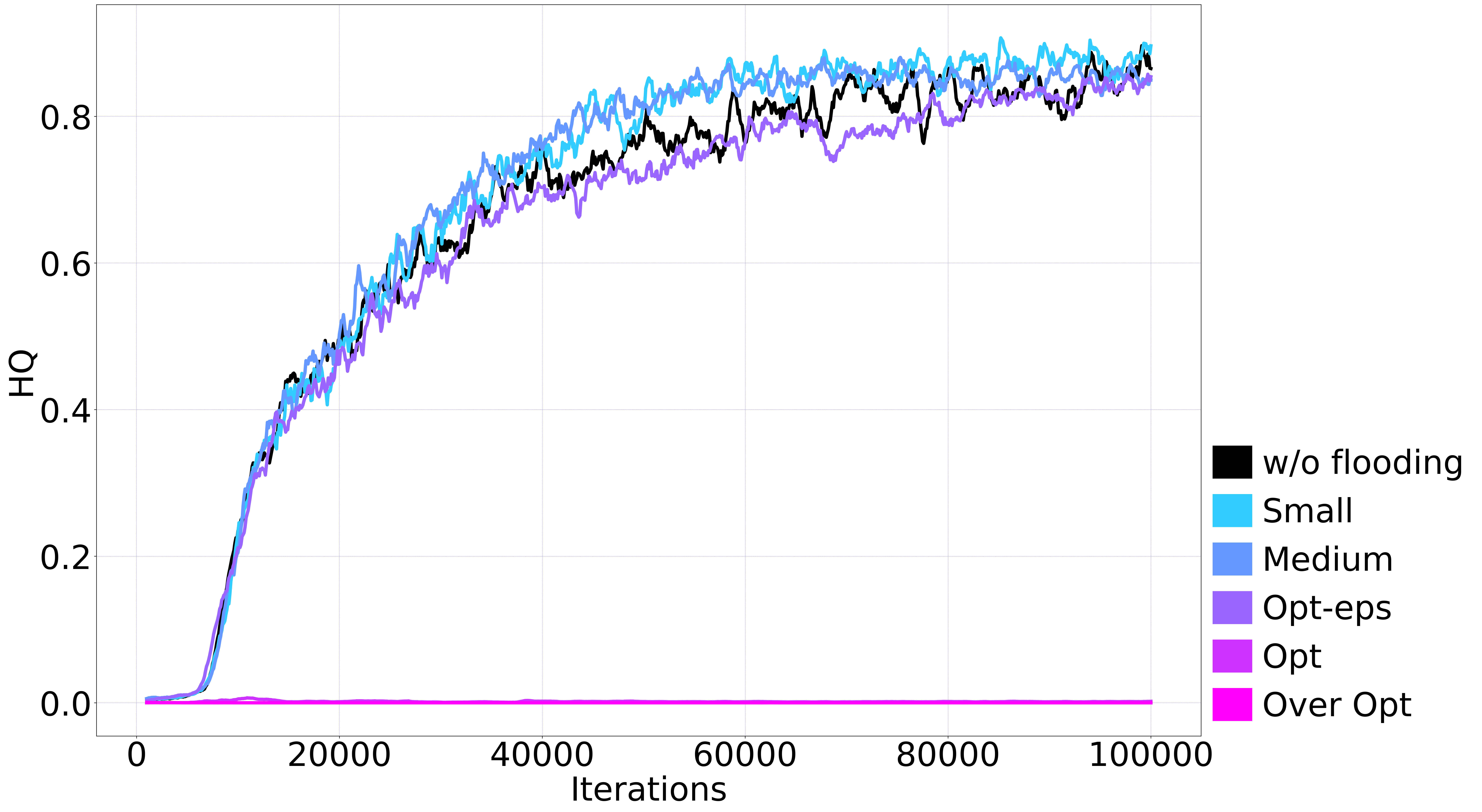}
  \subcaption{Transition of HQ}
 \end{minipage} 
\end{minipage} 
\begin{minipage}{0.6\linewidth}
    \includegraphics[keepaspectratio, width=\linewidth]{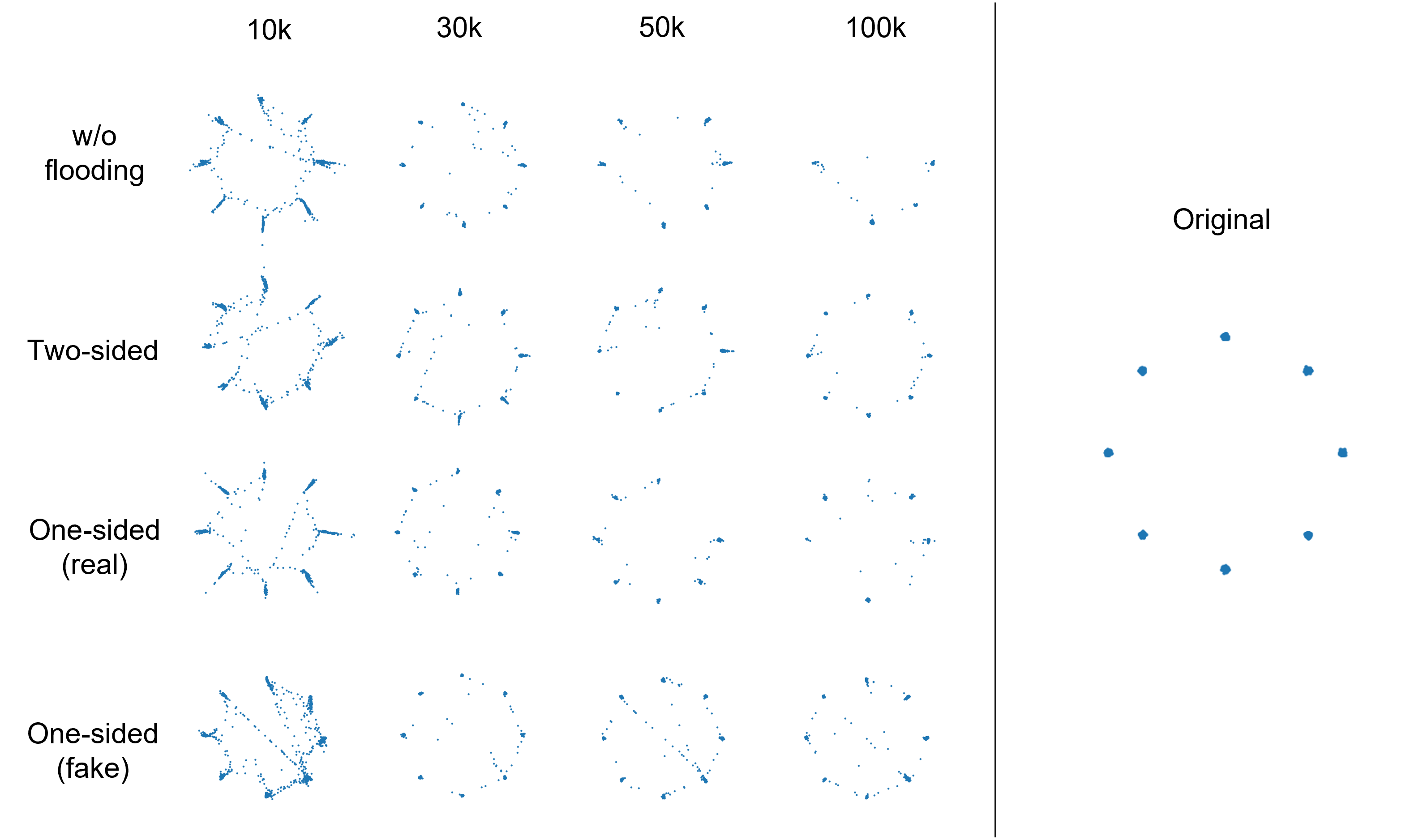}
    \subcaption{Generated samples}
\end{minipage}
\end{center}
\caption{
(a), (b):
Relationship between flood levels and modes and HQ transition with the moving average.
(c):
Transition of 1,000 generated samples. The horizontal axis indicates the iterations.
In the early phase (10k, 30k), the performance of all experiments is comparatively good. However, after the mid-phase (50k), without flooding or with one-sided (real) flooding, some modes disappear.
In contrast, with two-sided flooding or one-sided (fake) flooding, the training collapse is avoided.
}
\label{fig:synthetic-mode-hq-transition}
\end{figure}

\begin{table}[t]\small
  \begin{center}
    {\small{
\begin{tabular}{lccccccc}
\toprule
Adversarial loss & Eval & w/o flooding & Small & Medium & Near Opt & Opt & Over Opt \\
\midrule
BCE loss    & Modes & 4.8 (1.2) & 6.6 (0.8) & \textbf{7.8 (0.4)} & 7.0 (0.0) & 2.0 (0.6) & 0.0 (0.0) \\
            & HQ & 0.90 (0.11) & 0.94 (0.03) & \textbf{0.87 (0.07)} & 0.90 (0.03) & 0.00 (0.00) & 0.00 (0.00) \\
Hinge loss       & Modes & 7.4 (0.8) & 6.6 (0.8) & \textbf{8.0 (0.0)} & 7.4 (0.8) & 0.6 (0.5) & 0.0 (0.0) \\
            & HQ & 0.83 (0.13) & 0.73 (0.20) & \textbf{0.78 (0.04)} & 0.81 (0.07) & 0.00 (0.00) & 0.00 (0.00) \\
Least squares loss         & Modes & 6.6 (0.8) & 6.8 (0.4) & \textbf{7.8 (0.4)} & 7.4 (0.5) & 0.0 (0.0) & 0.0 (0.0) \\
            & HQ & 0.80 (0.09) & 0.79 (0.09) & \textbf{0.90 (0.02)} & 0.78 (0.07) & 0.00 (0.00) & 0.00 (0.00) \\
Wasserstein loss     & Modes & 8.0 (0.0) & \textbf{8.0 (0.0)} & \textbf{8.0 (0.0)} & 0.0 (0.0) & 0.0 (0.0) & 0.0 (0.0) \\
            & HQ & 0.93 (0.01) & \textbf{0.95 (0.01)} & \textbf{0.95 (0.01)} & 0.00 (0.00) & 0.00 (0.00) & 0.00 (0.00) \\
\bottomrule
\end{tabular}
}}
\end{center}
\caption{
Average (standard deviation) of modes and HQ with various adversarial losses and flood levels.
}
\label{table:result_synthetic_adversarial_losses}
\end{table}

\subsection{Flooding for various adversarial losses}
\label{sec:flooding_for_various_adv_losses}
We also tested the effect of flooding with other adversarial losses: the hinge loss, the least squares loss, and the Wasserstein loss with gradient penalty (WGAN-GP).
The flood level setting is provided in Table \ref{table:flooding_level}.

The results are shown in Table \ref{table:result_synthetic_adversarial_losses}.
It indicates that flooding with various adversarial losses and flood level Medium stabilizes the training. 

\begin{table}[t]\small
  \begin{center}
    {\small{
\begin{tabular}{lcccccc}
\toprule
Adversarial loss & Eval & w/o flooding & Two-sided & One-sided (real) & One-sided (fake) & Smoothing \\
\midrule
BCE Loss        & Modes & 4.8 (1.2) & 7.8 (0.4) & 4.0 (1.8) & \textbf{8.0 (0.0)} & 7.4 (0.8) \\
                & HQ & 0.90 (0.11) & 0.87 (0.07) & 0.77 (0.13) & \textbf{0.89 (0.08)} & 0.88 (0.06)\\
Hinge Loss      & Modes & 7.4 (0.8) & 8.0 (0.0) & 5.4 (1.5) & \textbf{8.0 (0.0)} & 7.4 (0.5)\\
                & HQ & 0.83 (0.13) & 0.78 (0.04) & 0.79 (0.05) & \textbf{0.81 (0.06)} & 0.87 (0.05)\\
Least squares loss  & Modes & 6.6 (0.8) & \textbf{7.8 (0.4)} & 5.6 (0.5) & 7.8 (0.4) & 2.6 (0.5)\\
                & HQ & 0.80 (0.09) & \textbf{0.90 (0.02)} & 0.83 (0.08) & 0.85 (0.12) & 0.92 (0.06)\\
Wasserstein loss         & Modes & 8.0 (0.0) & \textbf{8.0 (0.0)} & \textbf{8.0 (0.0)} & 8.0 (0.0) & 0.2 (0.4)\\
                & HQ & 0.93 (0.01) & \textbf{0.95 (0.01)} & \textbf{0.95 (0.01)} & 0.93 (0.02) & 0.00 (0.00)\\
                
\bottomrule
\end{tabular}
}}
\end{center}
\caption{
Average (standard deviation) of modes and HQ with various adversarial losses with one-sided or two-sided flooding with flood level Medium.
}
\label{table:result_synthetic_various_loss_one-sided}
\end{table}

\subsection{One-sided flooding}
\label{sec:one-sided_flooding}
In Section \ref{sec:flooding_bce} we confirmed the effect of flooding with flooding type $L_{D,\text{flood},1}$, where flooding is applied to the discriminator's loss for both real and generated data, we call it `two-sided flooding'.
As a special case, we investigated what happens if flooding was applied only to the loss for real data (`one-sided (real)' flooding) or generated data (`one-sided (fake)' flooding), i.e., when we set only $b_{\mathrm{real}}$ or $b_{\mathrm{fake}}$, respectively.
We did experiments on various adversarial losses with flood level Medium.

Table \ref{table:result_synthetic_various_loss_one-sided} shows the results.
We can see that the performance with one-sided (fake) flooding is comparable to that of two-sided flooding except for the Wasserstein loss.
It is noteworthy that one-sided (real) led to significant performance degradation in the BCE loss, the hinge loss, and the least-squares loss.
It suggests that the discriminator overfits the generated data rather than real data with the setting. 
In other words, although it is difficult for the discriminator to overfit real samples drawn from the true distribution at each iteration, it is easier to overfit the generated samples from the generator whose expressiveness is low.
In Figure~\ref{fig:synthetic-mode-hq-transition} (c), we show the transition of generated samples with the BCE loss.

\subsection{Comparison with label smoothing}

Next, we compared the effect of flooding and label smoothing.
Label smoothing and the proposed method share similarities as stabilization techniques that do not add regularization terms while preserving the type of adversarial losses. 
Label smoothing calculates $L_{D,\text{real}}$ with $a = 0.9$ (recommended by~\cite{NIPS2016_8a3363ab}),
\begin{equation}
\label{eq:smoothing}
\begin{split}
    L_{D,\text{real}} = \mathbb{E}_{\bm{x}\sim \mathbb{P}_r}[&a\cdot f_{D,\text{real}}(D(\bm{x})) +(1-a)\cdot f_{D,\text{fake}}(D(\bm{x}))].
\end{split}
\end{equation}

Table \ref{table:result_synthetic_various_loss_one-sided} shows that the label smoothing has a positive effect only for the BCE loss and the hinge loss, and it worsens the performance for the least squares loss and the Wasserstein loss while flooding performs well for all four losses.

We provide additional experiments in Supplementary Section \ref{sec:additional_synthetic}.

\subsection{Flooding for DCGAN}
\label{sec:flooding_for_DCGAN}

\begin{table}[t]\small
  \begin{center}
    {\small{
\begin{tabular}{lcccccc}
\toprule
(a) Adversarial loss & w/o flooding & Two-sided & One-sided (real)& One-sided (fake) \\
\midrule
BCE loss & 317.4 (129.0) & 236.2 (125.1) & \textbf{112.2 (145.1)} & 392.4 (46.5) \\ 
Hinge loss & 74.4 (33.6) & 237.5 (143.6) & \textbf{44.6 (3.2)} & 368.6 (37.8) \\
Least squares loss & 67.9 (34.8) & 204.8 (82.8) & \textbf{51.2 (13.3)} & 417.8 (109.0) \\
Wasserstein loss & 82.7 (13.5) & 402.6 (35.5) & \textbf{75.7 (13.8)} & 79.6 (3.1) \\
\midrule
(b) Regularization & & & &  \\
\midrule
Gradient Penalty & 214.0 (151.8) & 184.9 (147.9) & \textbf{62.5 (8.8)} & 195.5 (173.2) \\
Spectral Normalization & 270.5 (149.4) & 314.0 (56.9) & \textbf{34.9 (2.5)} & 290.8 (95.9) \\
\bottomrule
\end{tabular}
}}
\end{center}
\caption{
Results with CelebA (128$\times$128) with or without flooding for (a) each adversarial loss and (b) regularization with the BCE loss. 
}
\label{table:flooding_adv_loss_result_celebA128}
\end{table}

\begin{figure}[t]
 \begin{minipage}[b]{0.325\linewidth}
  \centering
  \includegraphics[keepaspectratio, width=0.9\linewidth]{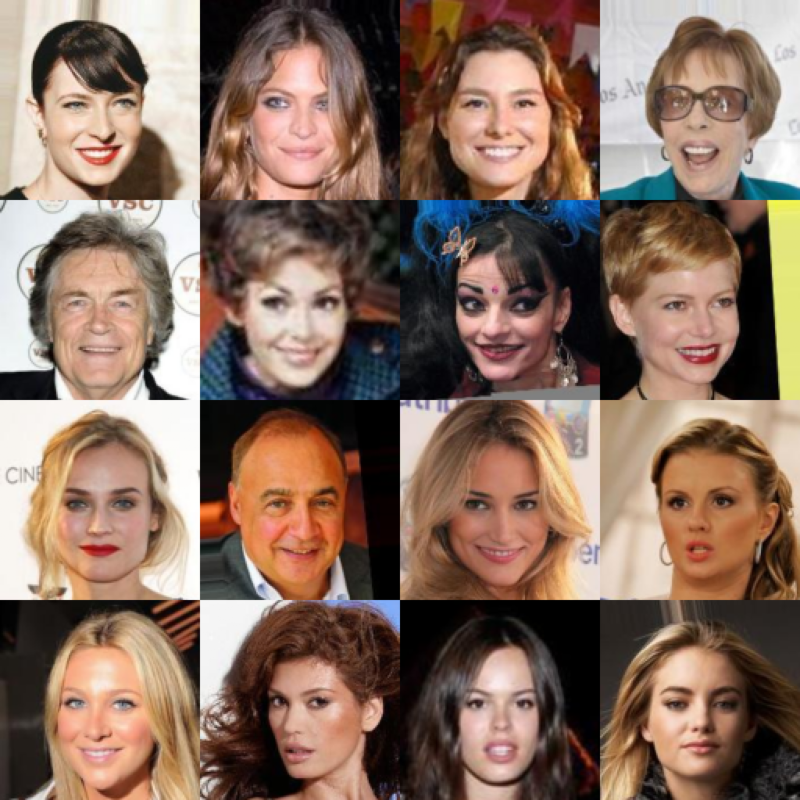}
  \subcaption{Training images}
 \end{minipage}
 \begin{minipage}[b]{0.325\linewidth}
  \centering
  \includegraphics[keepaspectratio, width=0.9\linewidth]{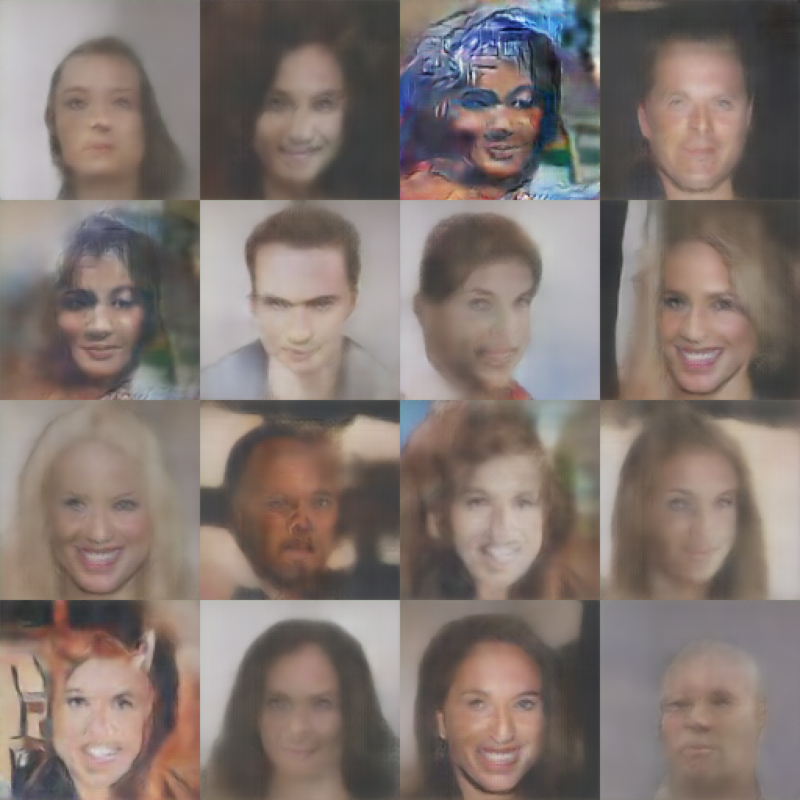}
  \subcaption{Without flooding}
 \end{minipage}
 \begin{minipage}[b]{0.325\linewidth}
  \centering
  \includegraphics[keepaspectratio, width=0.9\linewidth]{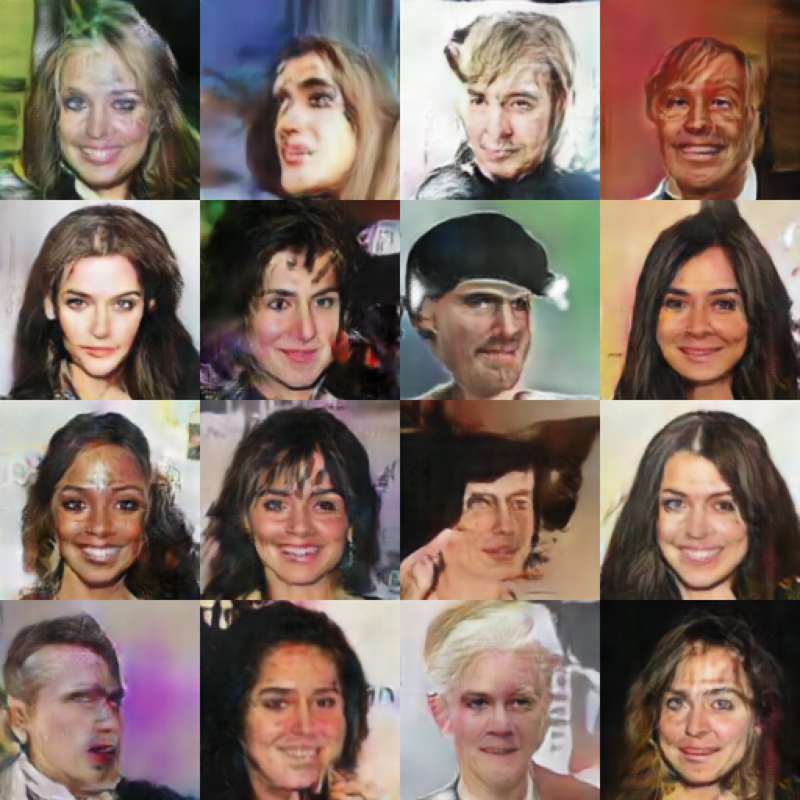}
  \subcaption{With flooding}
 \end{minipage}
 \caption{
 Comparison of generated images with CelebA (128$\times$128) and the BCE loss.
 Without flooding, training completely collapses in four out of five trials with generated images like noise. 
 In one trial, there was no such collapse, but compared to the training images (a), the generated image gradually collapses, for example, the graying ((b), FID: 60.5).
 In contrast, if flooding is applied at the appropriate flood level, the generated images of four trials are not collapsed, as shown in ((c), FID: 38.4).
 }
\label{fig:flooding_bce_images}
\end{figure}

\noindent\textbf{Flooding for various adversarial losses}
\label{sec:flooding_each_loss_exp}
Next, we examined the effect of flooding on image generation tasks.
We use various adversarial losses, DCGAN, and CelebA (128$\times$128) with flood level Medium.

We show the results in Table \ref{table:flooding_adv_loss_result_celebA128} (a) and the generated images in Figure \ref{fig:flooding_bce_images}.
The results show that the combination of one-sided (real) flooding and the change in the type of adversarial loss further stabilizes the training.
On the other hand, it seems that two-sided or one-sided (fake) is not effective in most cases.

\begin{table}[t]
  \begin{center}
    {\small{
\begin{tabular}{lcccc}
\toprule
Dataset & w/o flooding & Two-sided & One-sided (real) & One-sided (fake) \\
\midrule
CIFAR10 & 35.5 (0.8) & 37.2 (1.5) & \textbf{34.3 (0.8)} & 36.5 (1.0) \\
CIFAR100 & 41.7 (1.7) & 38.9 (1.4) & \textbf{36.2 (1.4)} & 41.8 (1.6)\\
STL10 & 154.1 (9.8) & \textbf{133.6 (4.9)} & \textbf{133.6 (3.4)} & 144.3 (13.0) \\
CelebA (64$\times$64) & 91.0 (2.7) & 89.8 (1.4) & \textbf{87.3 (1.1)} & 179.8 (105.6) \\
CelebA (128$\times$128) & 317.4 (129.0) & 236.2 (125.1) & \textbf{112.2 (145.1)} & 392.4 (46.5) \\
\bottomrule
\end{tabular}
}}
\end{center}
\caption{
FID for each dataset with or without flooding. 
}
\label{table:dataset}
\end{table}

\noindent\textbf{Flooding for various datasets}
Next, we verified the stabilization effect of flooding regardless of the dataset.
We conducted experiments on CIFAR10, CIFAR100, and STL10, with DCGAN, the BCE loss, and flood level Medium.

We show the results in Table \ref{table:dataset}.
Although the difficulty of the image generation task is different among the datasets, one-sided (real) flooding is effective on various datasets.
One reason why one-sided (real) flooding is effective is because it prevents the discriminator from memorization in GANs, hypothesized in Section 4.2 of~\cite{brock2018large}.
In other words, if the discriminator memorizes the limited dataset, its loss would sharply drop, but flooding can prevent it.

\subsection{Flooding with other techniques}
\label{sec:flooding_other_techniques}
Additional regularization terms and architectural changes can stabilize the GAN training.
We examined the effect of flooding in combination with spectral normalization and gradient penalty, commonly used as improvement methods, on DCGAN and CelebA (128$\times$128), with the BCE loss.
Note that we used gradient penalty with the BCE loss as `GAN-GP' in~\cite{miyato2018spectral}, while we used it with the Wasserstein loss in Section \ref{sec:flooding_each_loss_exp} , which is necessary for the theoretical proof of training convergence~\cite{NIPS2017_892c3b1c}.

Table \ref{table:flooding_adv_loss_result_celebA128} (b) shows the results.
Even when using the gradient penalty for the BCE loss or spectral normalization alone, they could not prevent the collapse of GAN training on the dataset.
On the other hand, using these regularization with flooding improves the performance, which shows the effect of flooding.

\begin{table}[t]
  \begin{center}
    {\small{
\begin{tabular}{lcccc}
\toprule
Application & w/o flooding & Two-sided & One-sided (real) & One-sided (fake) \\
\midrule
CDCGAN & 90.0 (19.6) & 88.7 (16.2) & \textbf{66.2 (4.9)} & 108.7 (14.3) \\
\midrule
ADDA & 0.60 (0.07) & \textbf{0.76 (0.05)} & 0.70 (0.06) & 0.24 (0.25) \\
DANN & \textbf{0.74 (0.02)} & 0.67 (0.05) & 0.72 (0.03) & 0.71 (0.02) \\
WDGRL & 0.65 (0.08) & 0.57 (0.07) & \textbf{0.67 (0.07)} & 0.66 (0.08) \\
\bottomrule
\end{tabular}
}}
\end{center}
\caption{
Results with or without in combination with various adversarial frameworks. 
We evaluate the performance of conditional generation (CDCGAN) with FID and domain adaptation (ADDA, DANN, and WDGRL) with accuracy for the target data.
}
\label{table:adversarial_result}
\end{table}

\subsection{Flooding with other adversarial application}
\label{sec:flooding_with_adv_app}
We examined the effect of flooding with other adversarial applications, such as conditional GANs~\cite{mirza2014conditional} and domain adaptation.
For the experiments of conditional GANs, we conduct experiments with conditional DCGAN (CDCGAN) and CIFAR10.
To verify the effect of flooding on adversarial learning for domain adaptation, we used ADDA~\cite{Tzeng_2017_CVPR}, DANN~\cite{dann}, and WDGRL~\cite{wdgrl}. We conducted experiments with MNIST (source domain) and MNIST-M (target domain).

Table \ref{table:adversarial_result} shows the result.
First, the results demonstrate that flooding performs well on CDCGAN.
Moreover, it shows that flooding was significantly effective on ADDA, but not on DANN and WDGRL. 
One possible interpretation is that the discriminator only in ADDA causes overfitting.
DANN and WDGRL update the model that generates features from the source domain data during the training, whereas ADDA fixes the model.
Therefore, it's possible that in ADDA, the discriminator overfits fixed source domain features from the fixed model, but flooding prevents it.
Therefore, it supports the effect of flooding in a wide range of adversarial applications to prevent overfitting and increase performance.

\subsection{Flooding with large models}
\label{sec:flooding_large_models}
We verified the effect of flooding in large-scale GANs.
StarGAN V2 ~\cite{choi2020stargan} demonstrates high performance in image-to-image translation.
The generator generates style codes from random latent codes or reference images, and then generates images from source images and the style codes.
We tried one-sided (real) flooding because it performs better for image generation in Section \ref{sec:flooding_for_DCGAN}.

Table \ref{table:stargan-v2} shows the performance.
One-sided (real) flooding improved three out of four metrics.
For the FID (reference) measurement where performance dropped, the variance is significantly larger than that of other measurements. Therefore, we believe its reliability is lower than the others.
Figure \ref{fig:flooding_stargan_images} shows examples of the generated images.

\begin{table}[t]
  \begin{center}
    {\small{
\begin{tabular}{lcccc}

\toprule
 & FID (latent) ($\downarrow$) & LPIPS (latent) ($\uparrow$) & FID (reference) ($\downarrow$) & LPIPS (reference) ($\uparrow$) \\
\midrule
w/o flooding & 13.92 (0.38) & 0.4494 (0.004) & \textbf{23.51 (1.59)} & 0.3886 (0.001) \\
Flooding & \textbf{13.23 (0.37)} & \textbf{0.4555 (0.003)} & 24.64 (0.99) & \textbf{0.3925 (0.003)} \\
\bottomrule
\end{tabular}
}}
\end{center}
\caption{
Results of the image generation  from random latent vector (latent) and reference images (reference) on StarGAN V2.
}
\label{table:stargan-v2}
\end{table}

\begin{figure}[t]
\begin{center}
  \includegraphics[keepaspectratio, width=0.9\linewidth]{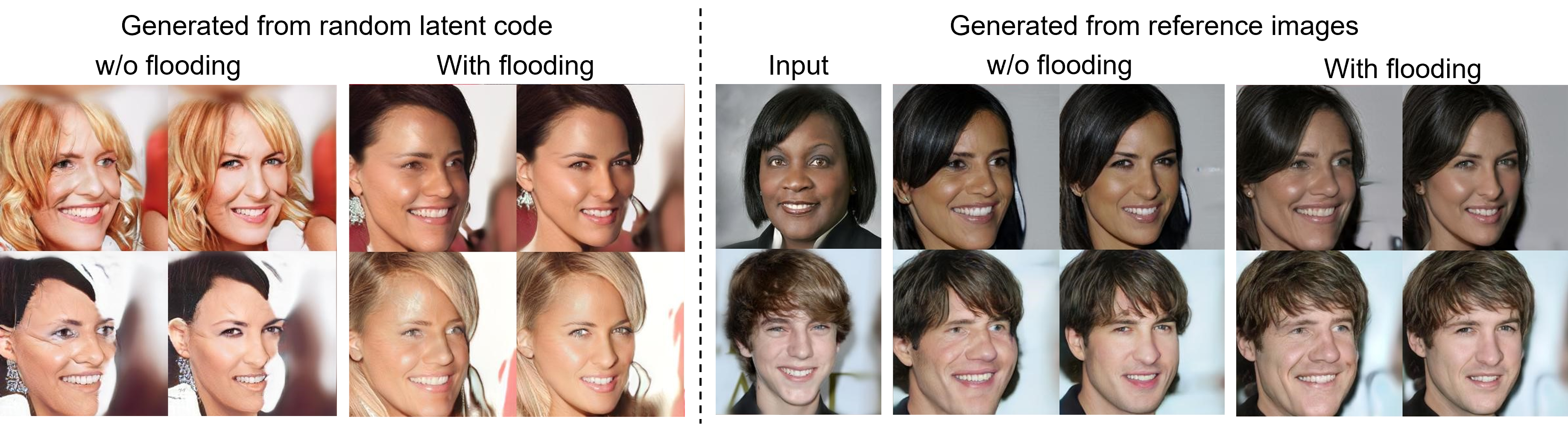}
\end{center}
 \caption{
 Generated images with StarGAN V2 and CelebAHQ.
 The latent vector or reference images are fixed in the rows, while the source images are fixed in the columns.
 }
\label{fig:flooding_stargan_images}
\end{figure}

\subsection{Limitations}
While flooding demonstrates stabilization effects in the training of various GANs, we observed collapsed results with CelebA (128$\times$128) on one for five runs.
Therefore, we should be aware that flooding cannot prevent all instabilities of GAN training.
For example, well-known existing methods, such as gradient penalty and spectral normalization ensure Lipschitz continuity of the discriminator; however, it is beyond the scope of our approach. 
Therefore, flooding should be combined with other stabilization techniques appropriately.

\section{Conclusion and future work}
We proposed to apply flooding, a method for preventing overfitting in supervised learning, to GANs.
Although our proposed method has an additional hyperparameter $b$, we demonstrated how we consider a range for the flood level.
We support the proposal through the theoretical analysis of the relationship between the flood level and the distribution of generated data.
The stabilization effect of flooding and the proposal's validity were demonstrated through experiments.
We also showed that flooding is effective when combined with existing training stabilization methods.

Further investigation is necessary to understand why GAN training with flooding can progress stably.

\section*{Acknowledgment} 
We appreciate Johannes Ackermann for reviewing the paper.

\bibliographystyle{splncs04}
\bibliography{main}

\begin{thebibliography}{10}
\providecommand{\url}[1]{\texttt{#1}}
\providecommand{\urlprefix}{URL }
\providecommand{\doi}[1]{https://doi.org/#1}

\bibitem{arjovsky2017towards}
Arjovsky, M., Bottou, L.: Towards principled methods for training generative adversarial networks. In: ICLR (2017)

\bibitem{pmlr-v70-arjovsky17a}
Arjovsky, M., Chintala, S., Bottou, L.: {W}asserstein generative adversarial networks. In: ICML (2017)

\bibitem{brock2018large}
Brock, A., Donahue, J., Simonyan, K.: Large scale {GAN} training for high fidelity natural image synthesis. In: ICLR (2019)

\bibitem{che2017mode}
Che, T., Li, Y., Jacob, A., Bengio, Y., Li, W.: Mode regularized generative adversarial networks. In: ICLR (2017)

\bibitem{choi2020stargan}
Choi, Y., Uh, Y., Yoo, J., Ha, J.W.: Stargan v2: Diverse image synthesis for multiple domains. In: CVPR (2020)

\bibitem{Demir2018PatchBasedII}
Demir, U., {\"U}nal, G.B.: Patch-based image inpainting with generative adversarial networks. arXiv preprint arXiv:1803.07422  (2018)

\bibitem{MDGANs}
Eghbal-zadeh, H., Zellinger, W., Widmer, G.: Mixture density generative adversarial networks. In: CVPR (2019)

\bibitem{dann}
Ganin, Y., Ustinova, E., Ajakan, H., Germain, P., Larochelle, H., Laviolette, F., March, M., Lempitsky, V.: Domain-adversarial training of neural networks. Journal of Machine Learning Research  \textbf{17}(59),  1--35 (2016)

\bibitem{NIPS2014_5ca3e9b1}
Goodfellow, I., Pouget-Abadie, J., Mirza, M., Xu, B., Warde-Farley, D., Ozair, S., Courville, A., Bengio, Y.: Generative adversarial nets. In: NeurIPS (2014)

\bibitem{NIPS2017_892c3b1c}
Gulrajani, I., Ahmed, F., Arjovsky, M., Dumoulin, V., Courville, A.C.: Improved training of wasserstein {GAN}s. In: NeurIPS (2017)

\bibitem{NIPS2017_8a1d6947}
Heusel, M., Ramsauer, H., Unterthiner, T., Nessler, B., Hochreiter, S.: Gans trained by a two time-scale update rule converge to a local nash equilibrium. In: NeurIPS (2017)

\bibitem{pmlr-v37-ioffe15}
Ioffe, S., Szegedy, C.: Batch normalization: Accelerating deep network training by reducing internal covariate shift. In: ICML (2015)

\bibitem{ishida2021need}
Ishida, T., Yamane, I., Sakai, T., Niu, G., Sugiyama, M.: Do we need zero training loss after achieving zero training error? In: ICML (2020)

\bibitem{pix2pix}
Isola, P., Zhu, J.Y., Zhou, T., Efros, A.A.: Image-to-image translation with conditional adversarial networks. In: CVPR (2017)

\bibitem{karras2018progressive}
Karras, T., Aila, T., Laine, S., Lehtinen, J.: Progressive growing of {GAN}s for improved quality, stability, and variation. In: ICLR (2018)

\bibitem{Karras_2019_CVPR}
Karras, T., Laine, S., Aila, T.: A style-based generator architecture for generative adversarial networks. In: CVPR (2019)

\bibitem{karras2020analyzing}
Karras, T., Laine, S., Aittala, M., Hellsten, J., Lehtinen, J., Aila, T.: Analyzing and improving the image quality of stylegan. In: CVPR (2020)

\bibitem{DBLP:journals/corr/KodaliAHK17}
Kodali, N., Abernethy, J.D., Hays, J., Kira, Z.: How to train your {DRAGAN}. CoRR  \textbf{abs/1705.07215} (2017), \url{http://arxiv.org/abs/1705.07215}

\bibitem{8099502}
Ledig, C., Theis, L., Huszar, F., Caballero, J., Cunningham, A., Acosta, A., Aitken, A., Tejani, A., Totz, J., Wang, Z., Shi, W.: Photo-realistic single image super-resolution using a generative adversarial network. In: CVPR (2017)

\bibitem{10.1145/3569928}
Li, Z., Usman, M., Tao, R., Xia, P., Wang, C., Chen, H., Li, B.: A systematic survey of regularization and normalization in gans. ACM Comput. Surv.  \textbf{55}(11) (2023)

\bibitem{lim2017geometric}
Lim, J.H., Ye, J.C.: Geometric gan. arXiv preprint arXiv:1705.02894  (2017)

\bibitem{Mao_2017_ICCV}
Mao, X., Li, Q., Xie, H., Lau, R.Y., Wang, Z., Paul~Smolley, S.: Least squares generative adversarial networks. In: ICCV (2017)

\bibitem{mirza2014conditional}
Mirza, M., Osindero, S.: Conditional generative adversarial nets. arXiv preprint arXiv:1411.1784  (2014)

\bibitem{miyato2018spectral}
Miyato, T., Kataoka, T., Koyama, M., Yoshida, Y.: Spectral normalization for generative adversarial networks. In: ICLR (2018)

\bibitem{park2019SPADE}
Park, T., Liu, M.Y., Wang, T.C., Zhu, J.Y.: Semantic image synthesis with spatially-adaptive normalization. In: CVPR (2019)

\bibitem{petzka2018on}
Petzka, H., Fischer, A., Lukovnikov, D.: On the regularization of wasserstein {GAN}s. In: ICLR (2018)

\bibitem{radford2016unsupervised}
Radford, A., Metz, L., Chintala, S.: Unsupervised representation learning with deep convolutional generative adversarial networks. arXiv preprint arXiv:1511.06434  (2016)

\bibitem{NIPS2016_8a3363ab}
Salimans, T., Goodfellow, I., Zaremba, W., Cheung, V., Radford, A., Chen, X., Chen, X.: Improved techniques for training {GANs}. In: NeurIPS (2016)

\bibitem{wdgrl}
Shen, J., Qu, Y., Zhang, W., Yu, Y.: Wasserstein distance guided representation learning for domain adaptation. AAAI'18/IAAI'18/EAAI'18, AAAI Press (2018)

\bibitem{Shorten2019}
Shorten, C., Khoshgoftaar, T.M.: A survey on image data augmentation for deep learning. Journal of Big Data  \textbf{6}(1), ~60 (2019)

\bibitem{VEEGANs}
Srivastava, A., Valkov, L., Russell, C., Gutmann, M.U., Sutton, C.: Veegan: Reducing mode collapse in gans using implicit variational learning. In: Guyon, I., Luxburg, U.V., Bengio, S., Wallach, H., Fergus, R., Vishwanathan, S., Garnett, R. (eds.) NeurIPS. vol.~30. Curran Associates, Inc. (2017)

\bibitem{JMLR:v15:srivastava14a}
Srivastava, N., Hinton, G., Krizhevsky, A., Sutskever, I., Salakhutdinov, R.: Dropout: A simple way to prevent neural networks from overfitting. Journal of Machine Learning Research  \textbf{15}(56),  1929--1958 (2014), \url{http://jmlr.org/papers/v15/srivastava14a.html}

\bibitem{rethinking}
Szegedy, C., Vanhoucke, V., Ioffe, S., Shlens, J., Wojna, Z.: Rethinking the inception architecture for computer vision. In: CVPR (2016)

\bibitem{cf-mc-gans}
Thanh-Tung, H., Tran, T.: Catastrophic forgetting and mode collapse in gans. In: IJCNN. pp. 1--10 (2020)

\bibitem{Tzeng_2017_CVPR}
Tzeng, E., Hoffman, J., Saenko, K., Darrell, T.: Adversarial discriminative domain adaptation. In: CVPR (2017)

\bibitem{10.1145/3439723}
Wang, Z., She, Q., Ward, T.E.: Generative adversarial networks in computer vision: A survey and taxonomy. ACM Comput. Surv.  \textbf{54}(2) (feb 2021)

\bibitem{xiang2017effects}
Xiang, S., Li, H.: On the effects of batch and weight normalization in generative adversarial networks. arXiv preprint arXiv:1704.03971  (2017)

\bibitem{xie2022iflood}
Xie, Y., WANG, Z., Li, Y., Zhang, C., Zhou, J., Ding, B.: iflood: A stable and effective regularizer. In: ICLR (2022)

\bibitem{pmlr-v97-zhang19d}
Zhang, H., Goodfellow, I., Metaxas, D., Odena, A.: Self-attention generative adversarial networks. In: ICML (2019)

\bibitem{stackgan}
Zhang, H., Xu, T., Li, H., Zhang, S., Wang, X., Huang, X., Metaxas, D.: Stackgan: Text to photo-realistic image synthesis with stacked generative adversarial networks. In: ICCV (2017)

\bibitem{zhang2018unreasonable}
Zhang, R., Isola, P., Efros, A.A., Shechtman, E., Wang, O.: The unreasonable effectiveness of deep features as a perceptual metric. In: CVPR (2018)

\bibitem{CycleGAN2017}
Zhu, J.Y., Park, T., Isola, P., Efros, A.A.: Unpaired image-to-image translation using cycle-consistent adversarial networks. In: ICCV (2017)

\end{thebibliography}

\newpage

\appendix

\section{Proof of Theorem 1}
\label{sec:proof_of_theorem_1}

\renewcommand{\proofname}{\textit{Proof}}
\begin{proof}
The discriminator's loss $L_{D,\text{flood},1}$ with the BCE loss is expressed as
\begin{eqnarray}
\begin{aligned}
L_{D,flood,1} =&\mathbb{E}_{\bm{x}\sim \mathbb{P}_r}[h(f_{D,\mathrm{real}}(D(\bm{x})), b_{\,\text{real}})] + \mathbb{E}_{\bm{x} \sim \mathbb{P}_g}[h(f_{D,\mathrm{fake}}(D(\bm{x})), b_{\,\text{fake}})] \\
=&\mathbb{E}_{\bm{x}\sim \mathbb{P}_r}[h(-\log(D(\bm{x})), b_{\,\text{real}})] + \mathbb{E}_{\bm{x} \sim \mathbb{P}_g}[h(-\log(1-D(\bm{x})), b_{\,\text{fake}})] \\
=&\int \mathbb{P}_r(\bm{x})[h(-\log(D(\bm{x})), b_{\,\text{real}})] +\mathbb{P}_g(\bm{x})[h(-\log(1-D(\bm{x})), b_{\,\text{fake}})] \mathrm{d}\bm{x}. \\
\end{aligned}
\end{eqnarray}
We introduce $f(D(\bm{x}))$ to examine the relationship between $D$ and $L_{D,flood,1}$ shown as
\begin{eqnarray}
\begin{aligned}
f(D(\bm{x})) =& \mathbb{P}_r(\bm{x})[h(-\log(D(\bm{x})), b_{\,\text{real}})] +\mathbb{P}_g(\bm{x})[h(-\log(1-D(\bm{x})), b_{\,\text{fake}})].
\end{aligned}
\end{eqnarray}
We demonstrate minimization of $f(D(\bm{x}))$ in two cases with respect to whether $e^{-b_{\,\text{real}}} + e^{-b_{\,\text{fake}}} \leq 1 $ is satisfied or not.

\noindent\textbf{Case 1.} If the flood level satisfies $e^{-b_{\,\text{real}}} + e^{-b_{\,\text{fake}}} \leq 1 $, $b_{\,\text{real}}$ and $b_{\,\text{fake}}$ satisfy
\begin{eqnarray}
\begin{aligned}
e^{-b_{\,\text{real}}} \leq 1 - e^{-b_{\,\text{fake}}}
\end{aligned}
\end{eqnarray}
We divide $D(\bm{x})\in (0,1)$ to three intervals with respect to $e^{-b_{\,\text{real}}}$ and $1 - e^{-b_{\,\text{fake}}}$.

\noindent\textbf{Case 1(a).} If we assume $0<D(\bm{x})\leq e^{-b_{\,\text{real}}}$, we can transform $f(D(\bm{x}))$ as
\begin{eqnarray}
\begin{aligned}
\label{eq:f_1-1}
f(D(\bm{x})) =& \mathbb{P}_r(\bm{x})[-\log(D(\bm{x}))] +\mathbb{P}_g(\bm{x})[\log(1-D(\bm{x}))+2b_{\,\text{fake}}]. \\
\end{aligned}
\end{eqnarray}
Therefore, the derivative with respect to $D(\bm{x})$ is expressed as
\begin{eqnarray}
\begin{aligned}
\label{eq:f_diff_1-1}
f'(D(\bm{x})) =& -\frac{\mathbb{P}_r(\bm{x})}{D(\bm{x})} - \frac{\mathbb{P}_g(\bm{x})}{1-D(\bm{x})}.
\end{aligned}
\end{eqnarray}
We obtain $f'(D(\bm{x})) < 0$ because of $(\mathbb{P}_r(\bm{x}), \mathbb{P}_g(\bm{x})) \in \mathbb{R}_{\geq 0}^2 \ \backslash \ \{(0,0)\} $ and $D(\bm{x}) \in (0,1)$.
Therefore, $f(D(\bm{x}))$ in the interval $0<D(\bm{x}) \leq e^{-b_{\,\text{real}}}$ monotonically decrease, and $D(\bm{x})=e^{-b_{\,\text{real}}}$ gives the minimum of $f(D(\bm{x}))$ in the interval.

\noindent\textbf{Case 1(b).} If we assume $e^{-b_{\,\text{real}}}\leq D(\bm{x})\leq 1-e^{-b_{\,\text{fake}}}$, we can obtain $f(D(\bm{x}))$ and $f'(D(\bm{x}))$ as 
\begin{eqnarray}
\begin{aligned}
f(D(\bm{x})) =& \mathbb{P}_r(\bm{x})[\log(D(\bm{x}))+2b_{\,\text{real}}] +\mathbb{P}_g(\bm{x})[\log(1-D(\bm{x}))+2b_{\,\text{fake}}], 
\ f'(D(\bm{x})) = \frac{\mathbb{P}_r(\bm{x})}{D(\bm{x})} - \frac{\mathbb{P}_g(\bm{x})}{1-D(\bm{x})}.
\end{aligned}
\end{eqnarray}
Now, we can obtain $D(\bm{x})$ where $f'(D(\bm{x}))=0$ as
\begin{eqnarray}
\begin{aligned}
 D(\bm{x}) =\frac{\mathbb{P}_r(\bm{x})}{\mathbb{P}_r(\bm{x})+\mathbb{P}_g(\bm{x})}.
\end{aligned}
\end{eqnarray}
Moreover, $f'(e^{-b_{\,\text{real}}})$ and $f'(1-e^{-b_{\,\text{fake}}})$ are expressed as
\begin{eqnarray}
\begin{aligned}
\label{eq:f_diff}
f'(e^{-b_{\,\text{real}}})=& \frac{\mathbb{P}_r(\bm{x})}{e^{-b_{\,\text{real}}}}-\frac{\mathbb{P}_g(\bm{x})}{1-e^{-b_{\,\text{real}}}} 
=\frac{(1-e^{-b_{\,\text{real}}})\mathbb{P}_r(\bm{x})-e^{-b_{\,\text{real}}}\mathbb{P}_g(\bm{x})}{e^{-b_{\,\text{real}}}(1-e^{-b_{\,\text{real}}})} \\
=&\frac{\mathbb{P}_r(\bm{x})-e^{-b_{\,\text{real}}}(\mathbb{P}_r(\bm{x}) + \mathbb{P}_g(\bm{x}))}{e^{-b_{\,\text{real}}}(1-e^{-b_{\,\text{real}}})},  \\
f'(1-e^{-b_{\,\text{fake}}})=& \frac{\mathbb{P}_r(\bm{x})}{1-e^{-b_{\,\text{fake}}}}-\frac{\mathbb{P}_g(\bm{x})}{e^{-b_{\,\text{fake}}}} 
=\frac{e^{-b_{\,\text{fake}}}\mathbb{P}_r(\bm{x})-(1-e^{-b_{\,\text{fake}}})\mathbb{P}_g(\bm{x})}{e^{-b_{\,\text{fake}}}(1-e^{-b_{\,\text{fake}}})} \\ 
=&\frac{\mathbb{P}_r(\bm{x})-(1-e^{-b_{\,\text{fake}}})(\mathbb{P}_r(\bm{x}) + \mathbb{P}_g(\bm{x}))}{e^{-b_{\,\text{fake}}}(1-e^{-b_{\,\text{fake}}})}. \\
\end{aligned}
\end{eqnarray}
If $\mathbb{P}_g(\bm{x})$ satisfies an inequality 
\begin{eqnarray}
\begin{aligned}
e^{-b_{\,\text{real}}} < \frac{\mathbb{P}_r(\bm{x})}{\mathbb{P}_r(\bm{x})+\mathbb{P}_g(\bm{x})} <1-e^{-b_{\,\text{fake}}},
\end{aligned}
\end{eqnarray}
we obtain $f'(e^{-b_{\,\text{real}}}) > 0$ and $f'(1-e^{-b_{\,\text{fake}}}) < 0$.
Therefore, $D(\bm{x}) \in \{ 1-e^{-b_{\,\text{fake}}}, e^{-b_{\,\text{real}}}\}$ gives the minimum.
If $\mathbb{P}_g(\bm{x})$ does not satisfy the inequality, $f(D(\bm{x}))$ monotonically increase or decrease in the interval, and $D(\bm{x}) \in \{ 1-e^{-b_{\,\text{fake}}}, e^{-b_{\,\text{real}}}\}$ gives the minimum.

\noindent\textbf{Case 1(c).} If we assume $1-e^{-b_{\,\text{fake}}}\leq D(\bm{x})<1$, we can obtain $f(D(\bm{x}))$ and $f'(D(\bm{x}))$ as 
\begin{eqnarray}
\begin{aligned}
f(D(\bm{x})) =& \mathbb{P}_r(\bm{x})[\log(D(\bm{x}))+2b_{\,\text{real}}] +\mathbb{P}_g(\bm{x})[-\log(1-D(\bm{x}))], 
\ f'(D(\bm{x})) = \frac{\mathbb{P}_r(\bm{x})}{D(\bm{x})} + \frac{\mathbb{P}_g(\bm{x})}{1-D(\bm{x})}.
\end{aligned}
\end{eqnarray}
As Case 1(a), we obtain $f'(D(\bm{x})) > 0$. 
Therefore, $f(D(\bm{x}))$ monotonically increase, and $D(\bm{x})=1-e^{-b_{\,\text{fake}}}$ gives the minimum of $f(D(\bm{x}))$ in the interval.

Finally, from Case 1(a), 1(b), and 1(c), we can prove Eq.~\eqref{Theorem_1_formula_1}.

\noindent\textbf{Case 2.} If the flood level satisfies $e^{-b_{\,\text{real}}} + e^{-b_{\,\text{fake}}} > 1 $, $b_{\,\text{real}}$ and $b_{\,\text{fake}}$ satisfy
\begin{eqnarray}
\begin{aligned}
  1-e^{-b_{\,\text{fake}}} < e^{-b_{\,\text{real}}},
\end{aligned}
\end{eqnarray}
and we can divide $(0,1)$ to three intervals with respect to $1-e^{-b_{\,\text{fake}}}$ and $e^{-b_{\,\text{real}}}$. 

\noindent\textbf{Case 2(a).} If we assume $0<D(\bm{x})\leq 1-e^{-b_{\,\text{fake}}}$, we have
\begin{eqnarray}
\begin{aligned}
f(D(\bm{x})) =& \mathbb{P}_r(\bm{x})[-\log(D(\bm{x}))] +\mathbb{P}_g(\bm{x})[\log(1-D(\bm{x}))+2b_{\,\text{fake}}], 
\ f'(D(\bm{x})) = -\frac{\mathbb{P}_r(\bm{x})}{D(\bm{x})} - \frac{\mathbb{P}_g(\bm{x})}{1-D(\bm{x})}.
\end{aligned}
\end{eqnarray}
As Case 1(a), we obtain  $f'(D(\bm{x}))<0$ in the interval.
Therefore, $D(\bm{x})=1-e^{-b_{\,\text{fake}}}$ gives the minimum of $f(D(\bm{x}))$ in the interval.

\noindent\textbf{Case 2(b).} If we assume $1-e^{-b_{\,\text{fake}}}\leq D(\bm{x}) \leq e^{-b_{\,\text{real}}}$, we have
\begin{eqnarray}
\begin{aligned}
f(D(\bm{x})) =& \mathbb{P}_r(\bm{x})[-\log(D(\bm{x}))] +\mathbb{P}_g(\bm{x})[-\log(1-D(\bm{x}))], 
\ f'(D(\bm{x})) = -\frac{\mathbb{P}_r(\bm{x})}{D(\bm{x})} + \frac{\mathbb{P}_g(\bm{x})}{1-D(\bm{x})}.
\end{aligned}
\end{eqnarray}
Therefore, $D(\bm{x}) =\displaystyle\frac{\mathbb{P}_r(\bm{x})}{\mathbb{P}_r(\bm{x})+\mathbb{P}_g(\bm{x})}$ gives $f'(D(\bm{x}))=0$ and satisfies $0 < D(\bm{x}) < 1$.
We can calculate $f'(1-e^{-b_{\,\text{fake}}})$ and $f'(e^{-b_{\,\text{real}}})$ as
\begin{eqnarray}
\begin{aligned}
\label{eq:f_diff_2}
f'(e^{-b_{\,\text{real}}})=& -\frac{\mathbb{P}_r(\bm{x})}{e^{-b_{\,\text{real}}}}+\frac{\mathbb{P}_g(\bm{x})}{1-e^{-b_{\,\text{real}}}} 
= -\frac{(1-e^{-b_{\,\text{real}}})\mathbb{P}_r(\bm{x})-e^{-b_{\,\text{real}}}\mathbb{P}_g(\bm{x})}{e^{-b_{\,\text{real}}}(1-e^{-b_{\,\text{real}}})} \\
=& -\frac{\mathbb{P}_r(\bm{x})-e^{-b_{\,\text{real}}}(\mathbb{P}_r(\bm{x}) + \mathbb{P}_g(\bm{x}))}{e^{-b_{\,\text{real}}}(1-e^{-b_{\,\text{real}}})},  \\
f'(1-e^{-b_{\,\text{fake}}})=& -\frac{\mathbb{P}_r(\bm{x})}{1-e^{-b_{\,\text{fake}}}}+\frac{\mathbb{P}_g(\bm{x})}{e^{-b_{\,\text{fake}}}} 
= -\frac{e^{-b_{\,\text{fake}}}\mathbb{P}_r(\bm{x})-(1-e^{-b_{\,\text{fake}}})\mathbb{P}_g(\bm{x})}{e^{-b_{\,\text{fake}}}(1-e^{-b_{\,\text{fake}}})} \\
=& -\frac{\mathbb{P}_r(\bm{x})-(1-e^{-b_{\,\text{fake}}})(\mathbb{P}_r(\bm{x}) + \mathbb{P}_g(\bm{x}))}{e^{-b_{\,\text{fake}}}(1-e^{-b_{\,\text{fake}}})}. \\
\end{aligned}
\end{eqnarray}
Now we obtain $f'(1-e^{-b_{\,\text{fake}}})<0$ and $f'(e^{-b_{\,\text{real}}})>0$ if $\mathbb{P}_g(\bm{x})$ satisfies an inequality
\begin{eqnarray}
\begin{aligned}
1-e^{-b_{\,\text{fake}}} < \frac{\mathbb{P}_r(\bm{x})}{\mathbb{P}_r(\bm{x})+\mathbb{P}_g(\bm{x})} < e^{-b_{\,\text{real}}}.
\end{aligned}
\end{eqnarray}
Therefore, $D(\bm{x}) = \displaystyle\frac{\mathbb{P}_r(\bm{x})}{\mathbb{P}_r(\bm{x})+\mathbb{P}_g(\bm{x})}$ gives the minimum in the interval.
On the other hand, if $\mathbb{P}_g(\bm{x})$ does not satisfy the inequality, $f(D(\bm{x}))$ monotonically increase or decrease, and $D(\bm{x}) \in \{ 1-e^{-b_{\,\text{fake}}}, e^{-b_{\,\text{real}}}\}$ gives the minimum in the interval.

\noindent\textbf{Case 2(c).} If we assume $e^{-b_{\,\text{real}}}\leq D(\bm{x})<1$, we have
\begin{eqnarray}
\begin{aligned}
f(D(\bm{x})) =& \mathbb{P}_r(\bm{x})[\log(D(\bm{x}))+2b_{\,\text{real}}] +\mathbb{P}_g(\bm{x})[-\log(1-D(\bm{x}))], 
\end{aligned}
\end{eqnarray}
As Case 1(a) we obtain $f'(D(\bm{x}))>0$.
Therefore, $D(\bm{x})=e^{-b_{\,\text{real}}}$ gives the minimum of $f(D(\bm{x}))$ in the interval.

Finally, from Case 2(a), 2(b), and 2(c), we can prove Eq.~\eqref{Theorem_1_formula_2}.
\qed
\end{proof}

\section{Discussions of Theorem 1}
\label{sec:analysis_of_theorem_1}
We can assume $b_{\,\text{fake}} > 0$ with flooding and the BCE loss because the BCE loss is non-negative.
If and only if $\mathbb{P}_g(\bm{x})$ satisfies the inequality \eqref{condition}, which assumes $e^{-b_{\,\text{real}}} + e^{-b_{\,\text{fake}}} > 1$,
\begin{equation}
\begin{split}
    \mathbb{P}_g(\bm{x}) \in ((\frac{1}{e^{-b_{\,\text{real}}}}-1)\mathbb{P}_r(\bm{x}), (\frac{1}{1-e^{-b_{\,\text{fake}}}}-1)\mathbb{P}_r(\bm{x})),
\end{split}
\end{equation}
on $Supp(\mathbb{P}_r(\bm{x}))$.
We refer to the interval as $I(\mathbb{P}_r(\bm{x}), b_{\,\text{real}}, b_{\,\text{fake}})$ and then obtain 
\begin{equation}
\begin{split}
    &b_{\,\text{real}, 0} \leq b_{\,\text{real}, 1} \Rightarrow I(\mathbb{P}_r(\bm{x}), b_{\,\text{real}, 0}, b_{\,\text{fake}})  \supseteq I(\mathbb{P}_r(\bm{x}), b_{\,\text{real}, 1}, b_{\,\text{fake}}), \\
    &b_{\,\text{fake}, 0} \leq b_{\,\text{fake}, 1} \Rightarrow I(\mathbb{P}_r(\bm{x}), b_{\,\text{real}}, b_{\,\text{fake}, 0})  \supseteq I(\mathbb{P}_r(\bm{x}), b_{\,\text{real}}, b_{\,\text{fake}, 1}).
\end{split}
\end{equation}
From this perspective, the discriminator with higher flood level is more difficult to satisfy the inequality \eqref{condition}.

Moreover, Theorem \ref{Theorem_1} indicates that even if either $b_{\,\text{real}}$ or $b_{\,\text{fake}}$ is greater than $\log 2$, there are some settings ($b_{\,\text{real}}$, $b_{\,\text{fake}}$) which satisfy $D^*(\bm{x})= \displaystyle\frac{\mathbb{P}_r(\bm{x})}{\mathbb{P}_r(\bm{x})+\mathbb{P}_g(\bm{x})}$ where $e^{-b_{\,\text{real}}} + e^{-b_{\,\text{fake}}} > 1$ and the inequality \eqref{condition}.
It can explain that some training results at Table \ref{table:result_flooding_various_setting_mode} and \ref{table:result_flooding_various_setting_hq} (detailed in Supplementary Section \ref{sec:additional_synthetic}) don't collapse completely, which is that the average of modes and HQ are not zero, with the BCE loss and the setting that either $b_{\,\text{real}}$ or $b_{\,\text{fake}}$ is greater than $\log 2$.
On the other hand, it is still unclear why the overall performance is poor with such setting .
We offer one explanation by using $I(\mathbb{P}_r(\bm{x}), b_{\,\text{real}}, b_{\,\text{fake}})$.
The ultimate goal of GANs is to achieve $\mathbb{P}_r(\bm{x}) = \mathbb{P}_g(\bm{x})$.
Therefore, we confirm that $I(\mathbb{P}_r(\bm{x}), b_{\,\text{real}}, b_{\,\text{fake}})$ includes $\mathbb{P}_r(\bm{x})$ on $Supp(\mathbb{P}_r(\bm{x}))$,
\begin{equation}
\begin{split}
    & \mathbb{P}_r(\bm{x}) = \mathbb{P}_g(\bm{x}) \in I(\mathbb{P}_r(\bm{x}), b_{\,\text{real}}, b_{\,\text{fake}}) \\
    & \Leftrightarrow ((\frac{1}{e^{-b_{\,\text{real}}}}-1) \mathbb{P}_r(\bm{x}) < \mathbb{P}_r(\bm{x})) \land (\mathbb{P}_r(\bm{x}) < (\frac{1}{1-e^{-b_{\,\text{fake}}}}-1) \mathbb{P}_r(\bm{x})) \\
    & \Leftrightarrow (\frac{1}{e^{-b_{\,\text{real}}}}-1 < 1) \land (1 < \frac{1}{1-e^{-b_{\,\text{fake}}}}-1) \\
    & \Leftrightarrow (e^{b_{\,\text{real}}} < 2) \land (e^{b_{\,\text{fake}}} < 2) \\
    & \Leftrightarrow (b_{\,\text{real}} < \log 2) \land (b_{\,\text{fake}} < \log 2).
\end{split}
\end{equation}
It indicates that the appropriate flood levels $b_{\,\text{real}}$ and $b_{\,\text{fake}}$ should be less than $L_{D_{\mathrm{opt}}, \mathrm{real}}$ and $L_{D_{\mathrm{opt}}, \mathrm{fake}}$, and the training does not converge to $\mathbb{P}_r(\bm{x}) = \mathbb{P}_g(\bm{x})$ if the flood level does not follow the appropriate setting.

\section{Code Implementation}
\label{sec:code_implementation}
We mainly referred to the following codes and customized them to conduct experiments with various losses and models.
\begin{itemize}
    \item Synthetic data: \url{https://github.com/igul222/improved_wgan_training}
    \item DCGAN: \url{https://github.com/pytorch/tutorials/blob/main/beginner_source/dcgan_faces_tutorial.py}
    \item FID: \url{https://github.com/mseitzer/pytorch-fid}
    \item ADDA: \url{https://github.com/jvanvugt/pytorch-domain-adaptation}
    \item Stargan-v2: \url{https://github.com/clovaai/stargan-v2}
\end{itemize}

\section{Details of Experiments}
\label{sec:details_of_Experiments}
\subsection{Flood level settings}
According to Hypotheses \ref{hyp:standard_1} and \ref{hyp:standard_2}, we consider that $L_{D_{\mathrm{opt}}, \mathrm{real}}$, $L_{D_{\mathrm{opt}}, \mathrm{fake}}$, and $L_{D_{\mathrm{opt}}}$, which is the discriminator's loss at the theoretical convergence, is crucial on the flood level $b$ setting.
To explore the optimal setting of the flood level, we verified the flood levels in Table \ref{table:flooding_level}.
For the flood level $b_{\,\text{real}}$ in the flooding types 1 and 2, we tried five different settings of the flood level in Table \ref{table:flooding_level}, and we use $b_{\,\text{fake}}$ calculated from Table \ref{table:flooding_level} which replace $L_{D_{\mathrm{opt}}, \mathrm{real}}$ to $L_{D_{\mathrm{opt}}, \mathrm{fake}}$. 
We examined the flood level $b_{\,\text{all}}$ in flooding type 3 for five different flood level settings in Table \ref{table:flooding_level}.
Note that for the Wasserstein loss, we cannot use the flood level strategy as other adversarial losses like $m \cdot L_{D_{\mathrm{opt}}}$ because of $L_{D_{\mathrm{opt}}}=0$ for the Wasserstein loss.
Therefore, we set an appropriate value while preserve the flood level Opt equals to $0$ ($=L_{D_{\mathrm{opt}}}$).

\begin{table}[t]
  \begin{center}
    {\small{
\begin{tabular}{llllll}
\toprule
Flooding type & Small & Medium & Near Opt & Opt & Over Opt \\
\midrule
\multicolumn{6}{l}{BCE loss and Least squares loss} \\
\hline
1, 2& $0.1 \cdot L_{D_{\mathrm{opt}}, \mathrm{real}}$ & $0.5 \cdot L_{D_{\mathrm{opt}}, \mathrm{real}}$ & $0.9 \cdot L_{D_{\mathrm{opt}}, \mathrm{real}}$ & $L_{D_{\mathrm{opt}}, \mathrm{real}}$ & $1.5 \cdot L_{D_{\mathrm{opt}}, \mathrm{real}}$ \\
3 & $0.1 \cdot L_{D_{\mathrm{opt}}}$ & $0.5 \cdot L_{D_{\mathrm{opt}}}$ & $0.9 \cdot L_{D_{\mathrm{opt}}}$ & $L_{D_{\mathrm{opt}}}$ & $1.5 \cdot L_{D_{\mathrm{opt}}}$ \\
\hline
\multicolumn{6}{l}{Hinge loss} \\
\hline
1, 2& $0.05 \cdot L_{D_{\mathrm{opt}}}$ & $0.25 \cdot L_{D_{\mathrm{opt}}}$ & $0.45 \cdot L_{D_{\mathrm{opt}}}$ & $0.5 \cdot L_{D_{\mathrm{opt}}}$ & $0.75 \cdot L_{D_{\mathrm{opt}}}$ \\
3 & $0.1 \cdot L_{D_{\mathrm{opt}}}$ & $0.5 \cdot L_{D_{\mathrm{opt}}}$ & $0.9 \cdot L_{D_{\mathrm{opt}}}$ & $L_{D_{\mathrm{opt}}}$ & $1.5 \cdot L_{D_{\mathrm{opt}}}$ \\
\hline
\multicolumn{6}{l}{Wasserstein loss} \\
\hline
1,2 &$-0.50$ & $-0.25$ & $-0.05$ & $0.0$ & $0.25$ \\
3 & $-1.0$ & $-0.5$ & $-0.1$ & $0.0$ &$0.5$ \\

\bottomrule
\end{tabular}
}}
\end{center}
\caption{
Flood level settings for pairs of adversarial losses and flooding type. 
$L_{D_{\mathrm{opt}}, \mathrm{real}}$ is the discriminator's loss for real data at the theoretical convergence, and $L_{D_{\mathrm{opt}}}$ is the sum of the discriminator's loss for real and generated data at the theoretical convergence.
When we used flooding for generated data, we replaced $L_{D_{\mathrm{opt}}, \mathrm{real}}$ with $L_{D_{\mathrm{opt}}, \mathrm{fake}}$.
}
\label{table:flooding_level}
\end{table}

\subsection{Implementation}

\noindent\textbf{Synthetic Dataset} 
To examine the effect of flooding for training GANs, we used the ring of 2D Gaussian dataset (2D Ring) as previous research~\cite{lim2017geometric,MDGANs,NIPS2017_892c3b1c,cf-mc-gans,Mao_2017_ICCV}.
The dataset is sampled from the distribution that is composed of eight Gaussian components with the same standard deviation $\sigma$, arranged in a circular pattern, as Figure \ref{fig:synthetic-mode-hq-transition} (c) (original).
For each iteration, we sampled training data from the distribution.
We use multi-layer perceptron as the generator and discriminator.
After training, we evaluated the variety and quality of generated samples with `modes' and `high quality (HQ),' proposed in~\cite{VEEGANs}.
We sampled 2,500 generated samples, then we counted modes as the number of the center of the Gaussian components that have a sample located within 3$\sigma$ in $L_2$ distance.
Furthermore, we calculated HQ as the ratio of the samples that have a center located within 3$\sigma$.
For instance, if we use 2D Ring, it holds that modes $\leq 8$ and HQ $\leq 1$.
In order to confirm how much our method prevents mode collapse, we consider higher modes as better performance. If the modes are same, we consider one with higher HQ as better performance.

\noindent\textbf{DCGAN} 
We used unconditional DCGAN~\cite{radford2016unsupervised} to evaluate the performance on image generation. 
We tried generating CIFAR10, CIFAR100 at 32$\times$32, STL10 at 64$\times$64, CelebA at both 64$\times$64 and 128$\times$128.
We followed Radford et al.~\cite{radford2016unsupervised} for batch normalization layers and Miyato et al.~\cite{miyato2018spectral} for spectral normalization layers.
When adding the gradient penalty, the implementation details, such as loss weights, also followed the research of Gulrajani et al.~\cite{NIPS2017_892c3b1c}
When we used CDCGAN, we split the first convolution layer of DCGAN to two convolution layers for the input noise and the class label.
After that, we concatenated output of the two convolution layers.
We used the same layers as DCGAN for the second and the following layers, and gave the second layer the concatenated output as input.
The batch size was 128, the learning rate was 0.0002, and Adam optimizer ($\beta = (0.5,0.999)$) was employed.
The training used one GPU for 100,000 iterations.
We conducted each experiment five times, and the generated images were evaluated using Fréchet Inception Distance (FID)~\cite{NIPS2017_8a1d6947}.

\noindent\textbf{Domain adaptation}
When we adapt flooding for ADDA~\cite{Tzeng_2017_CVPR}, DANN~\cite{dann}, and WDGRL~\cite{wdgrl}, we regard features from source domain as real distribution and one from target domain as generated distribution.
Note that the experimental results without flooding have deteriorated compared to the officially announced scores in the page as shown in Section \ref{sec:code_implementation}.
However, the source code even with no changes caused it, and there are issues about the reproducibility of other experiments at the repository page.
Therefore, we regarded that the score is not an average but the best score.

\noindent\textbf{Large model}
We conducted experiments using StarGAN V2~\cite{choi2020stargan} to investigate the effect of flooding in the generation of larger images.
We followed the author's implementation and used CelebAHQ.
The training used two GPUs for 100,000 iterations. We conducted each experiment five times and evaluated the generated images with FID and LPIPS~\cite{zhang2018unreasonable}. 

\section{Additional synthetic dataset experiments}
\label{sec:additional_synthetic}
We verified the effect of flooding in Section \ref{sec:Experiment}.
In this section, We conducted a further experiments with synthetic dataset to explore the potential effects.

\subsection{Flooding with various combination of $b_{\,\text{real}}$ and $b_{\,\text{fake}}$}

\begin{table}[t]\small
  \begin{center}
    {\small{
\begin{tabular}{lccccccc}
\toprule
$\downarrow b_{\,\text{real}}$  $b_{\,\text{fake}} \rightarrow$  & w/o flooding & Small  &  Medium & Near Opt & Opt & Over Opt \\
\midrule
w/o flooding & 4.8 (1.2) & 5.0 (2.6) & \textbf{8.0 (0.0)} & 8.0 (0.0) & 8.0 (0.0) & 8.0 (0.0) \\
Small & 2.4 (2.1) & 6.6 (0.8) & 8.0 (0.0) & 8.0 (0.0) & 8.0 (0.0) & 8.0 (0.0) \\
Medium & 4.0 (1.8) & 5.6 (0.5) & 7.8 (0.4) & 8.0 (0.0) & 8.0 (0.0) & 8.0 (0.0) \\
Near Opt & 2.2 (1.3) & 5.0 (0.0) & 5.6 (0.5) & 7.0 (0.0) & 7.8 (0.4) & \color{blue}0.0 (0.0)\color{black} \\
Opt & 2.0 (1.7) & 4.2 (0.4) & 5.0 (0.6) & 6.8 (0.4) & 2.0 (0.6) & \color{blue}0.0 (0.0)\color{black} \\
Over Opt & 1.4 (0.8) & 3.6 (0.8) & 5.0 (0.6) & \color{blue}0.0 (0.0)\color{black} & \color{blue}0.0 (0.0)\color{black} & \color{blue}0.0 (0.0)\color{black} \\
\bottomrule
\end{tabular}
}}
\end{center}
\caption{
Average (standard deviation) of modes with BCE loss with various flooding setting. 
We use BCE loss and $L_{D,\text{flood},1}$. We colorize value with blue if the  $b_{\,\text{real}}$ and $b_{\,\text{fake}}$ don't satisfy $1 < e^{-b_{\,\text{real}}} + e^{-b_{\,\text{fake}}}$.
}
\label{table:result_flooding_various_setting_mode}
\end{table}

\begin{table}[t]\small
  \begin{center}
    {\small{
\begin{tabular}{lccccccc}
\toprule
$\downarrow b_{\,\text{real}}$  $b_{\,\text{fake}} \rightarrow$  & w/o flooding & Small &  Medium & Near Opt & Opt & Over Opt \\
\midrule
w/o flooding & 0.90 (0.11) & 0.72 (0.36) & \textbf{0.89 (0.05)} & 0.82 (0.10) & 0.87 (0.09) & 0.51 (0.05)
 \\
Small & 0.53 (0.43) & 0.94 (0.03) & 0.86 (0.05) & 0.87 (0.06) & 0.88 (0.08) & 0.49 (0.12)
 \\
Medium & 0.77 (0.13) & 0.87 (0.05) & 0.87 (0.07) & 0.82 (0.06) & 0.82 (0.14) & 0.50 (0.09)
 \\
Near Opt & 0.69 (0.37) & 0.91 (0.04) & 0.84 (0.07) & 0.90 (0.03) & 0.83 (0.08) & \color{blue}0.00 (0.00)\color{black}
 \\
Opt & 0.53 (0.44) & 0.90 (0.06) & 0.87 (0.09) & 0.80 (0.12) & 0.00 (0.00) & \color{blue}0.00 (0.00)\color{black}
 \\
Over Opt & 0.71 (0.36) & 0.86 (0.12) & 0.74 (0.13) & \color{blue}0.00 (0.00)\color{black} & \color{blue}0.00 (0.00)\color{black} & \color{blue}0.00 (0.00)\color{black}
 \\
\bottomrule
\end{tabular}
}}
\end{center}
\caption{
Average (standard deviation) of HQ with BCE loss with various flooding setting. 
We use BCE loss and $L_{D,\text{flood},1}$. We colorize value with blue if the  $b_{\,\text{real}}$ and $b_{\,\text{fake}}$ don't satisfy $1 < e^{-b_{\,\text{real}}} + e^{-b_{\,\text{fake}}}$.
}
\label{table:result_flooding_various_setting_hq}
\end{table}

\begin{table}[t]\small
  \begin{center}
    {\small{
\begin{tabular}{lccccccc}
\toprule
$\downarrow b_{\,\text{real}}$  $b_{\,\text{fake}} \rightarrow$  & w/o flooding & Small &  Medium & Near Opt & Opt & Over Opt \\
\midrule
w/o flooding & - & - & - & - & - & - \\
Small & - & 1.87 & 1.64 & 1.47 & 1.43 & 1.29 \\
Medium & - & 1.64 & 1.41 & 1.24 & 1.21 & 1.06 \\
Near Opt & - & 1.47 & 1.24 & 1.07 & 1.04 & \color{blue}0.89\color{black} \\
Opt & - & 1.43 & 1.21 & 1.04 & 1.00 & \color{blue}0.85\color{black} \\
Over Opt & - & 1.29 & 1.06 & \color{blue}0.89\color{black} & \color{blue}0.85\color{black} & \color{blue}0.71\color{black} \\
\bottomrule
\end{tabular}
}}
\end{center}
\caption{
Table of $e^{-b_{\,\text{real}}} + e^{-b_{\,\text{fake}}}$ of BCE loss.
We colorize value with blue if the  $b_{\,\text{real}}$ and $b_{\,\text{fake}}$ don't satisfy $1 < e^{-b_{\,\text{real}}} + e^{-b_{\,\text{fake}}}$.
}
\label{table:inequality_is_held}
\end{table}

We investigated the effect of flooding for GAN training with the condition $b_{\,\text{real}} = b_{\,\text{fake}}$ or one-sided flooding in Section \ref{sec:flooding_bce} and \ref{sec:one-sided_flooding}.
In this section, we conducted experiments with $b_{\,\text{real}}$ and $b_{\,\text{fake}}$ without such condition. 
We assigned the six flood level settings (w/o flooding, Small,  Medium, Near Opt, Opt, and Over Opt) to $b_{\,\text{real}}$ and $b_{\,\text{fake}}$, and we conducted experiments with the 6$\times$6 combinations of $b_{\,\text{real}}$ and $b_{\,\text{fake}}$.

We show the experimental results of modes on Table \ref{table:result_flooding_various_setting_mode} and HQ on Table \ref{table:result_flooding_various_setting_hq}.
First, one-sided (fake) flooding ($b_{\,\text{real}}$: w/o flooding, $ b_{\,\text{fake}}$: Medium) shows the best performance.
It supports that the discriminator overfits the generated data with synthetic dataset.
Furthermore, it is noteworthy that if both $b_{\,\text{real}}$ and $b_{\,\text{fake}}$ exceed Opt, the training completely collapses, which is that the average of modes and HQ are zero.
On the other hand, if either of them exceeds Opt, some results does not completely collapse.
We give one explanation for the difference by referring to the inequality $e^{-b_{\,\text{real}}} + e^{-b_{\,\text{fake}}} \leq 1$ in Theorem \ref{Theorem_1}.  
We shows the values of $e^{-b_{\,\text{real}}} + e^{-b_{\,\text{fake}}}$ for each combination of $b_{\,\text{real}}$ and $b_{\,\text{fake}}$ in Table \ref{table:inequality_is_held}. 
Next, based on these values, we colored tables \ref{table:result_flooding_various_setting_mode} and \ref{table:result_flooding_various_setting_hq} blue where the settings that did not satisfy the inequality.
Consequently, the settings that did not satisfy the inequality corresponded to those where the training completely collapsed, it supports the arguments of Theorem \ref{Theorem_1}.
On the other hand, if either of $b_{\,\text{real}}$ and $b_{\,\text{fake}}$ exceeds Opt, the performance is still low performance in modes or HQ.
It supports the arguments of Supplementary Section \ref{sec:analysis_of_theorem_1}.

\subsection{Change of flooding type with various adversarial losses}
We verified the relation of flooding types ($L_{D,\text{flood},1}$, $L_{D,\text{flood},2}$, and $L_{D,\text{flood},3}$) and the performance with BCE loss in Section \ref{sec:flooding_bce}.
In this section, we investigated the best flooding type with adversarial losses other than BCE loss. 
The flood level setting is provided on Table \ref{table:flooding_level}.

The results are shown in Table \ref{table:result_synthetic_flooding_type_various_losses}.
With all adversarial losses $L_{D,\text{flood},1}$ flood level Medium achieved the best performance. 
For adversarial losses other than the Wasserstein loss, the upper and lower bounds for losses at the theoretical convergence $L_{D_{\mathrm{opt}}, \mathrm{real}}$ and $L_{D_{\mathrm{opt}}, \mathrm{fake}}$ are somewhat fixed. 
Therefore, we discovered an empirical rule that we should apply flooding with $L_{D,\text{flood},1}$ and the flood level within a range that is not too close to the upper and lower bounds.

\begin{table}\small
  \begin{center}
    {\small{
\begin{tabular}{lccccccc}
\toprule
Flooding type & Eval & w/o flooding & Small & Medium & Near Opt & Opt & Over Opt \\
\midrule
BCE Loss & & & & & & & \\
\midrule
$L_{D,\text{flood},1}$ & Modes & 4.8 (1.2) & 6.6 (0.8) & \textbf{7.8 (0.4)} & 7.0 (0.0) & 2.0 (0.6) & 0.0 (0.0) \\
                       & HQ & 0.90 (0.11) & 0.94 (0.03) & \textbf{0.87 (0.07)} & 0.90 (0.03) & 0.00 (0.00) & 0.00 (0.00) \\
$L_{D,\text{flood},2}$ & Modes & - & 4.0 (2.0) & 4.2 (1.3) & 7.0 (0.6) & 0.2 (0.4) & 0.0 (0.0) \\
                       & HQ & - & 0.65 (0.35) & 0.75 (0.18) & 0.28 (0.18) & 0.00 (0.00) & 0.00 (0.00) \\
$L_{D,\text{flood},3}$ & Modes & - & 4.0 (2.3) & 5.2 (1.9) & 7.2 (0.4) & 2.8 (1.6) & 0.0 (0.0) \\
                       & HQ & - & 0.68 (0.35) & 0.91 (0.06) & 0.24 (0.19) & 0.01 (0.01) & 0.00 (0.00) \\
\midrule
Hinge Loss & & & & & & & \\
\midrule
$L_{D,\text{flood},1}$ & Modes & 7.4 (0.8) & 6.6 (0.8) & \textbf{8.0 (0.0)} & 7.4 (0.8) & 0.6 (0.5) & 0.0 (0.0) \\
                       & HQ & 0.83 (0.13) & 0.73 (0.20) & \textbf{0.78 (0.04)} & 0.81 (0.07) & 0.00 (0.00) & 0.00 (0.00) \\
$L_{D,\text{flood},2}$ & Modes & - & 7.2 (0.4) & 6.6 (0.5) & 7.6 (0.5) & 0.0 (0.0) & 0.0 (0.0) \\
                       & HQ & - & 0.85 (0.10) & 0.83 (0.06) & 0.33 (0.16) & 0.00 (0.00) & 0.00 (0.00) \\
$L_{D,\text{flood},3}$ & Modes & - & 5.8 (0.7) & 7.6 (0.5) & 7.6 (0.5) & 2.4 (0.8) & 0.0 (0.0) \\
                       & HQ & - & 0.90 (0.03) & 0.79 (0.12) & 0.53 (0.31) & 0.00 (0.00) & 0.00 (0.00) \\
\midrule
LS Loss & & & & & & & \\
\midrule
$L_{D,\text{flood},1}$ & Modes & 6.6 (0.8) & 6.8 (0.4) & \textbf{7.8 (0.4)} & 7.4 (0.5) & 0.0 (0.0) & 0.0 (0.0) \\
                       & HQ & 0.80 (0.09) & 0.79 (0.09) & \textbf{0.90 (0.02)} & 0.78 (0.07) & 0.00 (0.00) & 0.00 (0.00) \\
$L_{D,\text{flood},2}$ & Modes & - & 7.0 (0.0) & 7.2 (0.7) & 6.8 (0.4) & 0.6 (0.8) & 0.2 (0.4) \\
                       & HQ & - & 0.82 (0.07) & 0.79 (0.13) & 0.44 (0.14) & 0.00 (0.00) & 0.00 (0.00) \\
$L_{D,\text{flood},3}$ & Modes & - & 6.3 (0.8) & 6.4 (0.5) & 6.8 (0.4) & 2.8 (1.5) & 0.0 (0.0) \\
                       & HQ & - & 0.87 (0.04) & 0.80 (0.13) & 0.40 (0.11) & 0.01 (0.01) & 0.00 (0.00) \\
\midrule
Wasserstein loss & & & & & & &  \\
\midrule
$L_{D,\text{flood},1}$ & Modes & 8.0 (0.0) & 8.0 (0.0) & \textbf{8.0 (0.0)} & 0.0 (0.0) & 0.0 (0.0) & 0.0 (0.0) \\
                       & HQ & 0.93 (0.01) & 0.95 (0.01) & \textbf{0.95 (0.01)} & 0.00 (0.00) & 0.00 (0.00) & 0.00 (0.00) \\
$L_{D,\text{flood},2}$ & Modes & - & 8.0 (0.0) & 8.0 (0.0) & 0.0 (0.0) & 0.0 (0.0) & 0.0 (0.0) \\
                       & HQ & - & 0.93 (0.02) & 0.94 (0.02) & 0.00 (0.00) &  0.00 (0.00) & 0.00 (0.00) \\
$L_{D,\text{flood},3}$ & Modes & - & 8.0 (0.0) & 7.6 (0.8) & 8.0 (0.0) & 0.8 (1.6) & 0.0 (0.0) \\
                       & HQ & - & 0.90 (0.06) & 0.91 (0.10) & 0.88 (0.06) & 0.02 (0.05) & 0.00 (0.00) \\
\bottomrule
\end{tabular}
}}
\end{center}
\caption{
Average (standard deviation) of modes and HQ with various adversarial losses, different flooding types, and flood levels.
}
\label{table:result_synthetic_flooding_type_various_losses}
\end{table}

\begin{table}
  \begin{center}
    {\small{
\begin{tabular}{lccccc}
\toprule
Flooding for $G$ ($\downarrow$) and $D$ ($\rightarrow$) & Eval & w/o flooding & Two-sided & One-sided real & One-sided fake \\
\midrule
w/o flooding    & Modes & 4.8 (1.2) & 7.8 (0.4) & 4.0 (1.8) & \textbf{8.0 (0.0)} \\
                & HQ & 0.90 (0.11) & 0.87 (0.07) & 0.77 (0.13) & \textbf{0.89 (0.08)} \\
With flooding   & Modes & 3.2 (2.6) & 7.6 (0.5) & 3.8 (0.7) & \textbf{8.0 (0.0)} \\
                & HQ & 0.46 (0.39) & 0.89 (0.06) & 0.83 (0.12) & \textbf{0.89 (0.07)} \\
\bottomrule
\end{tabular}
}}
\end{center}
\caption{
Average (standard deviation) of modes and HQ with or without flooding for the generator $G$ (vertical axis). We conducted experiments with BCE loss, flood level Medium, and flooding for the discriminator $D$ (horizontal axis).
}
\label{table:result_flooding_generator}
\end{table}

\subsection{Flooding for the generator's loss}
We verified the effect of flooding for the discriminator in previous sections because the discriminator often cause the instability.
In this section, we investigated the performance when flooding is applied to the generator. 
As flooding for the discriminator's loss with flooding type $L_{D,\text{flood},1}$, we adopt the following recalculation function with flood level $b_{\,G}$ as 
\begin{equation}
\begin{split}
\label{eq:GANs_adv_loss_flood_gen}
L_{G,\text{flood},1} 
=& \mathbb{E}_{x \sim \mathbb{P}_g}[h(f_{G}(D(\bm{x})), b_{\,\text{G}})] .
\\
\end{split}
\end{equation}
We conducted experiments with BCE loss and flood level $0.5 \cdot  \log 2$ because of $L_{G, \mathrm{opt}} = \log 2$ with BCE loss.

Table \ref{table:result_flooding_generator} shows the result. 
In this experimental setting, there was no benefit to applying flooding to the generator, in some cases it led to a degradation in performance.
However, since there could be potential benefits to applying flooding to the generator in experimental settings where the generator is prone to overfitting, we believe the potential of flooding for the generator.

\subsection{Change of experimental settings with flooding}
Next, we evaluated performance when changing experimental settings from the code described in Supplementary Section \ref{sec:code_implementation}.
We did experiments with changing of the number of updates, batch size, depth of the layers, and dataset size.

We show the results in Table \ref{table:result_synthetic_flooding_various_setting}.
It is important that the experiment results without flooding cause performance degradation in many cases according to the changes.
It indicates the vulnerability of GANs to changes in experimental settings.
On the other hand, both two-sided flooding and one-sided (fake) flooding generally improved the performance.
From these findings, we can conclude that flooding contributes to robustness against changes in experimental settings.

\begin{table}\small
  \begin{center}
    {\small{
\begin{tabular}{lccccc}
\toprule
Flooding type & Eval & w/o flooding & Two-sided & One-sided real & One-sided fake \\
\midrule
\multicolumn{6}{l}{(a) Number of updates for the generator or discriminator. } \\
\hline
Baseline & Modes & 4.8 (1.2) & 7.8 (0.4) & 4.0 (1.8) & \textbf{8.0 (0.0)} \\
                       & HQ & 0.90 (0.11) & 0.87 (0.07) & 0.77 (0.13) & \textbf{0.89 (0.08)} \\
$G$ updates (5$\times$) & Modes & 0.0 (0.0) & 7.0 (0.9) & 1.2 (0.4) & \textbf{8.0 (0.0)} \\
                       & HQ & 0.00 (0.00) & 0.87 (0.05) & 0.79 (0.17) & \textbf{0.82 (0.08)} \\
$D$ updates (5$\times$) & Modes & 0.0 (0.0) & \textbf{4.4 (1.0)} & 0.0 (0.0) & 0.0 (0.0) \\
                       & HQ & 0.00 (0.00) & \textbf{0.84 (0.09)} & 0.00 (0.00) & 0.00 (0.00) \\
\hline
\multicolumn{6}{l}{(b) Batch size. } \\
\hline
4 & Modes & 3.0 (2.5) & \textbf{6.6 (0.8)} & 1.0 (2.0) & 5.0 (2.7) \\
                       & HQ & 0.24 (0.21) & \textbf{0.50 (0.09)} & 0.11 (0.21) & 0.36 (0.22) \\
16 & Modes & 2.2 (1.8) & 6.6 (0.5) & 1.8 (2.2) & \textbf{7.4 (0.5)} \\
                       & HQ & 0.44 (0.36) & 0.74 (0.02) & 0.22 (0.27) & \textbf{0.79 (0.07)} \\
64 & Modes & 2.6 (2.2) & 7.2 (0.4) & 1.6 (2.1) & \textbf{8.0 (0.0)} \\
                       & HQ & 0.49 (0.40) & 0.84 (0.08) & 0.35 (0.43) & \textbf{0.77 (0.22)} \\
256 (default) & Modes & 4.8 (1.2) & 7.8 (0.4) & 4.0 (1.8) & \textbf{8.0 (0.0)} \\
                       & HQ & 0.90 (0.11) & 0.87 (0.07) & 0.77 (0.13) & \textbf{0.89 (0.08)} \\
512 & Modes & 2.4 (2.9) & 8.0 (0.0) & 5.0 (1.1) & \textbf{8.0 (0.0)} \\
                       & HQ & 0.30 (0.39) & 0.85 (0.06) & 0.90 (0.06) & \textbf{0.91 (0.02)} \\
\hline
\multicolumn{6}{l}{(c) Number of layers. } \\
\hline
2 & Modes & 6.2 (3.1) & \textbf{8.0 (0.0)} & 7.8 (0.4) & 6.4 (3.2) \\
                       & HQ & 0.24 (0.16) & \textbf{0.47 (0.12)} & 0.38 (0.13) & 0.30 (0.20) \\
4 (default) & Modes & 4.8 (1.2) & 7.8 (0.4) & 4.0 (1.8) & \textbf{8.0 (0.0)} \\
                       & HQ & 0.90 (0.11) & 0.87 (0.07) & 0.77 (0.13) & \textbf{0.89 (0.08)} \\
6 & Modes & 0.0 (0.0) & \textbf{6.4 (0.5)} & 0.0 (0.0) & 4.6 (3.8) \\
                       & HQ & 0.00 (0.00) & \textbf{0.78 (0.12)} & 0.00 (0.00) & 0.30 (0.32) \\
8 & Modes & 0.0 (0.0) & \textbf{5.8 (1.2)} & 0.0 (0.0) & 3.2 (3.9) \\
                       & HQ & 0.00 (0.00) & \textbf{0.61 (0.28)} & 0.00 (0.00) & 0.32 (0.39) \\
\hline
\multicolumn{6}{l}{(d) Dataset size. } \\
\hline
1000 & Modes & 0.0 (0.0) & 7.6 (0.5) & 0.0 (0.0) & \textbf{8.0 (0.0)} \\
                       & HQ & 0.00 (0.00) & 0.87 (0.06) & 0.00 (0.00) & \textbf{0.90 (0.04)} \\
10000 & Modes & 1.0 (1.3) & 8.0 (0.0) & 3.8 (2.5) & \textbf{8.0 (0.0)} \\
                       & HQ & 0.28 (0.36) & 0.75 (0.19) & 0.73 (0.37) & \textbf{0.86 (0.07)} \\
100000 & Modes & 2.2 (2.0) & \textbf{8.0 (0.0)} & 4.4 (1.4) & 8.0 (0.0) \\
                       & HQ & 0.56 (0.46) & \textbf{0.93 (0.01)} & 0.81 (0.13) & 0.87 (0.05) \\
$\infty$ (default) & Modes & 4.8 (1.2) & 7.8 (0.4) & 4.0 (1.8) & \textbf{8.0 (0.0)} \\
                       & HQ & 0.90 (0.11) & 0.87 (0.07) & 0.77 (0.13) & \textbf{0.89 (0.08)} \\
\bottomrule
\end{tabular}
}}
\end{center}
\caption{
Average (standard deviation) of modes and HQ with BCE loss with flood level Medium, various experimental settings, and different flooding styles. 
(a) $G$ updates (5$\times$) and $D$ updates (5$\times$) means that we increase the number of updates of the generator and the discriminator, respectively, from one to five for each iteration.
(d) The dataset size in the default setting is denoted as $\infty$ because each iteration in the default setting samples from the true distribution. 
On the other hand, when we use datasets of finite size, we first prepare a fixed number of data from the true distribution to use as the dataset.
}
\label{table:result_synthetic_flooding_various_setting}
\end{table}

\begin{table}
  \begin{center}
    {\small{
\begin{tabular}{lccccc}
\toprule
Dataset & Eval & w/o flooding & Two-sided & One-sided (real) & One-sided (fake) \\
\midrule
2D Ring & Modes & 4.8 (1.2) & 7.8 (0.4) & 4.0 (1.8) & \textbf{8.0 (0.0)} \\
                       & HQ & 0.90 (0.11) & 0.87 (0.07) & 0.77 (0.13) & \textbf{0.89 (0.08)} \\
2D Grid & Modes & 10.6 (5.5) & 19.6 (0.8) & 12.8 (1.5) & \textbf{21.6 (2.2)} \\
                       & HQ & 0.70 (0.35) & 0.85 (0.06) & 0.77 (0.12) & \textbf{0.87 (0.03)} \\
\bottomrule
\end{tabular}
}}
\end{center}
\caption{
Average (standard deviation) of modes and HQ with BCE loss with flood level Medium, datasets (2D Ring and 2D Grid), and different flooding styles.
}
\label{table:result_synthetic_dataset_25Gaussians}
\end{table}

\subsection{Change of datasets}
Next, we evaluated the effect of flooding on datasets other than 2D Ring. 
2D Grid is a dataset where the Gaussian centers are arranged in a 5$\times$5 grid. 
Note that it holds that modes $\leq 25$ because of changes of number of the Gaussian centers from 2D Ring.
We also evaluate the effect of flooding on this dataset with BCE loss and flood level Medium.

The results are presented in Table \ref{table:result_synthetic_dataset_25Gaussians}. 
Both two-sided flooding and one-sided (fake) flooding demonstrated performance improvement as 2D Ring.
It supports the flooding effect regardless of the dataset.

\subsection{Comparison with other functions with flooding function $h$}
The flooding function is defined as
\begin{equation}
\begin{split}
    h(L,b)=|L-b|+b=
    \begin{split}
      \begin{cases}
        -L+2b & \text{if\: $L < b$ is satisfied,} \\
        L & \text{otherwise.}
      \end{cases}
    \end{split}
\end{split}
\end{equation}
On the other hand, there are other possible methods than flooding to manipulate the loss below a certain value $b$.
We tried some methods (max, log, 10\%) that replace the flooding function $h$ with the following functions,
\begin{equation}
\begin{split}
\label{eq:GANs_other_modification}
h_{\text{max}}(L,b) =&
    \begin{split}
      \begin{cases}
        b & \text{if\: $L < b$ is satisfied,} \\
        L & \text{otherwise,}
      \end{cases}
    \end{split} \\
h_{\text{log}}(L,b) =&
    \begin{split}
      \begin{cases}
        b-\log(1+b-y) & \text{if\: $L < b$ is satisfied,} \\
        L & \text{otherwise,}
      \end{cases}
    \end{split}\\
h_{\text{10\%}}(L,b) =&
    \begin{split}
      \begin{cases}
        b-0.1 \cdot (b-y) & \text{if\: $L < b$ is satisfied,} \\
        L & \text{otherwise.}
      \end{cases}
    \end{split}\\
\end{split}
\end{equation}
The effect of the functions is illustrated in Figure \ref{fig:other_function}.
We applied the functions ($h_{\text{max}}$, $h_{\text{log}}$, and $h_{\text{10\%}}$) to the discriminator's loss.

\begin{figure}[t]
\begin{center}
\includegraphics[keepaspectratio, width=0.4\linewidth]{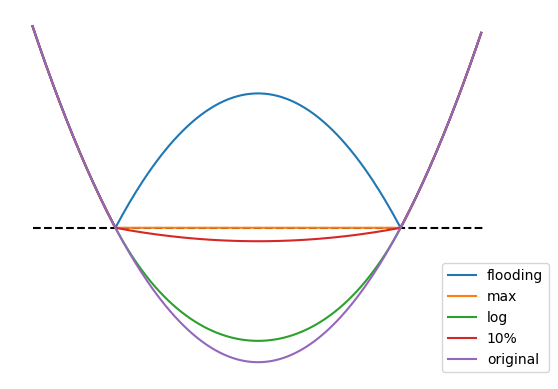}
\end{center}
\caption{
Illustration of changes of a graph (original) by using flooding (blue), max (yellow), log (green), and 10\% (red).
The horizontal dotted line (black) shows the $b$.
}
\label{fig:other_function}
\end{figure}

\begin{table}\small
  \begin{center}
    {\small{
\begin{tabular}{lccccccc}
\toprule
Adversarial loss & Eval & w/o flooding & flooding & max & log & 10\% \\
\midrule
BCE Loss    & Modes & 4.8 (1.2) & \textbf{7.8 (0.4)} & 4.0 (1.4) & 4.8 (0.7) & 5.2 (1.2) \\
            & HQ & 0.90 (0.11) & \textbf{0.87 (0.07)} & 0.90 (0.11) & 0.80 (0.12) & 0.85 (0.09) \\
                
\bottomrule
\end{tabular}
}}
\end{center}
\caption{
Average (standard deviation) of modes and HQ with BCE loss with flood level Medium, and functions which change the loss below a certain value $b$ as Figure \ref{fig:other_function}.
}
\label{table:result_synthetic_modification}
\end{table}

Table \ref{table:result_synthetic_modification} shows the results.
Results with the functions, $h_{\text{max}}$, $h_{\text{log}}$, and $h_{\text{10\%}}$, did not show meaningful improvements and some methods cause performance degradation.
It supports the advantage of flooding function $h$ in preventing the training instability.
It is important that even $h_{\text{max}}$ cannot prevent the training instability.
Both of $h_{\text{max}}$ and flooding function $h$ prevents the drop from $b$ as Figure \ref{fig:other_function}, however, only $h$ prevents the instability.
Therefore, we found that the effect of $h$ that causes the gradient flipping is crucial to GAN training.

\begin{figure}[t]
 \begin{minipage}[b]{0.48\linewidth}
  \centering
  \includegraphics[keepaspectratio, width=0.9\linewidth]{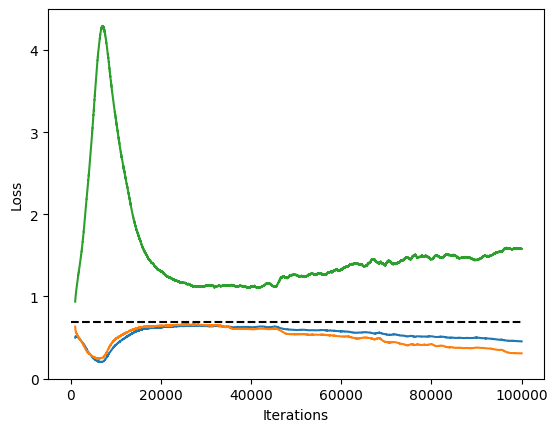}
  \subcaption{Without flooding}
 \end{minipage}
 \begin{minipage}[b]{0.48\linewidth}
  \centering
  \includegraphics[keepaspectratio, width=0.9\linewidth]{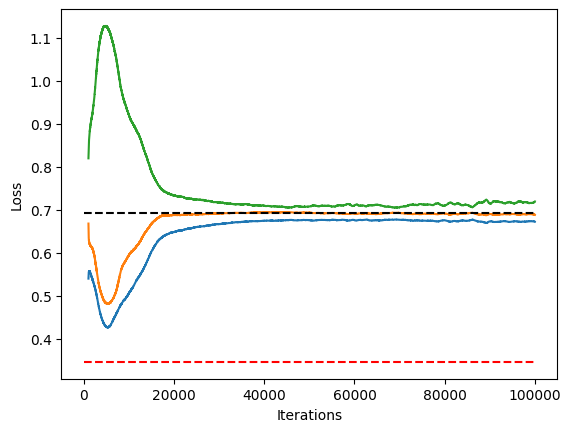}
  \subcaption{Two-sided flooding}
 \end{minipage}
 \\
 \begin{minipage}[b]{0.48\linewidth}
  \centering
  \includegraphics[keepaspectratio, width=0.9\linewidth]{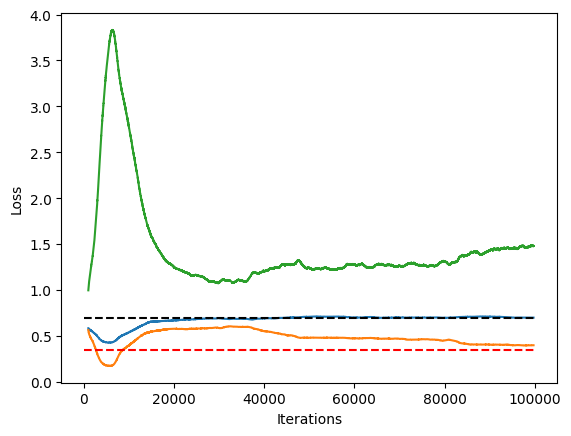}
  \subcaption{One-sided (real) flooding}
 \end{minipage} 
 \begin{minipage}[b]{0.48\linewidth}
  \centering
  \includegraphics[keepaspectratio, width=0.9\linewidth]{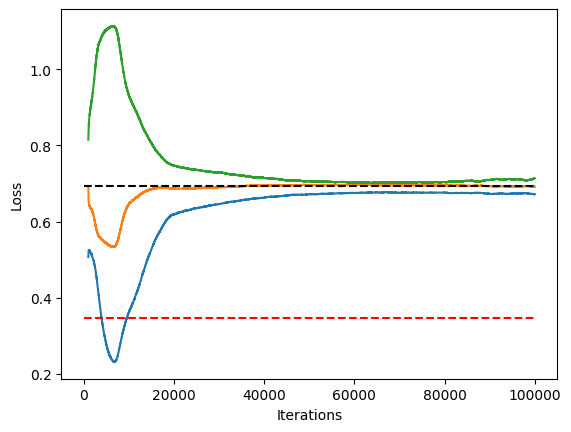}
  \subcaption{One-sided (fake) flooding}
 \end{minipage}
 \caption{Transition of losses $L_{D, \mathrm{real}}$ (blue), $L_{D, \mathrm{fake}}$ (yellow), and $L_{G}$ (green). 
 Because the fluctuations of these measurements during the training diminish the visibility, we took moving average with a window size of 1000. 
 The dashed line (black) expressed $L_{D, \mathrm{opt,real}}$ (blue), $L_{D, \mathrm{opt,fake}}$ (yellow), and $L_{G, \mathrm{opt}}$ ($=\log 2$). 
 The dashed line (red) expresses the flood level for the discriminator. Note that the discriminator's loss shown in the figure is that before applying flooding.
 }
 \label{fig:synthetic-loss-transition}
\end{figure}

\begin{figure}[t]
 \begin{minipage}[b]{0.48\linewidth}
  \centering
  \includegraphics[keepaspectratio, width=0.9\linewidth]{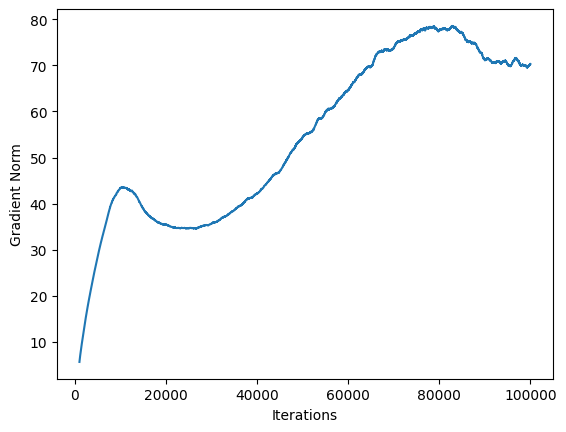}
  \subcaption{Without flooding}
 \end{minipage}
 \begin{minipage}[b]{0.48\linewidth}
  \centering
  \includegraphics[keepaspectratio, width=0.9\linewidth]{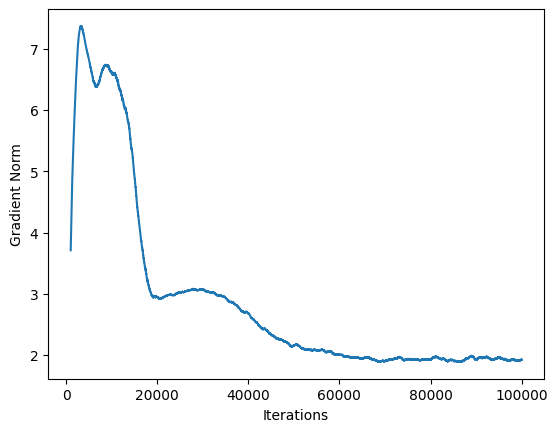}
  \subcaption{Two-sided flooding}
 \end{minipage}
 \\
 \begin{minipage}[b]{0.48\linewidth}
  \centering
  \includegraphics[keepaspectratio, width=0.9\linewidth]{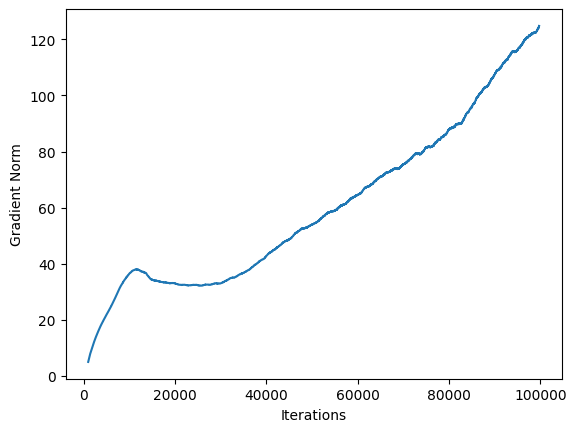}
  \subcaption{One-sided (real) flooding}
 \end{minipage} 
 \begin{minipage}[b]{0.48\linewidth}
  \centering
  \includegraphics[keepaspectratio, width=0.9\linewidth]{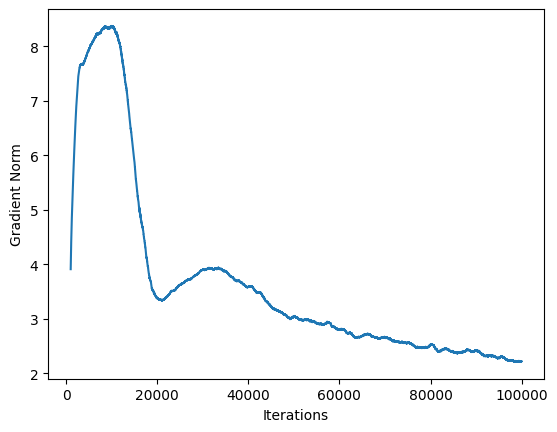}
  \subcaption{One-sided (fake) flooding}
 \end{minipage}
 \caption{
 Transition of gradient of the discriminator. 
 }
 \label{fig:synthetic-gradient-transition}
\end{figure}

\section{Analysis on the loss and gradient}
\label{sec:analysis_loss_gradient}
We showed experimentally that flooding stabilizes GAN training.
However, it is unknown why GAN training with flooding succeeds well even when flooding prevents the discriminator's loss from becoming low. 
For instance, the flood level Medium is not too low, but GAN training with the flood level was stable rather than collapsed.
In this section, We examined the loss of the generator and the discriminator during the training with synthetic dataset, BCE loss, and $L_{D,\mathrm{flood},1}$ to understand the dynamics of the two models training.

Figure \ref{fig:synthetic-loss-transition} shows the results.
Without flooding (Figure \ref{fig:synthetic-loss-transition} (a)), the discriminator's loss ($L_{D, \mathrm{real}}$ and $L_{D, \mathrm{fake}}$) is low around 10,000 iterations, while the generator's loss $L_{G}$ is high.
This is because the generated samples are not accurate at the beginning of training, making it easier to distinguish the real or generated samples, resulting in a lower discriminator's loss.
Subsequently, as the quality of the generated samples improves, the discriminator's loss increases steadily, but around 40,000 iterations, the discriminator's loss begins to decrease again and then it continues until the end of the training.
The loss transitions can be interpreted as overfitting of the discriminator to either the real or generated samples. 
On the other hand, with two-sided flooding (Figure \ref{fig:synthetic-loss-transition} (b)), such collapse does not occur, and it can be observed that it stably converges to the values at theoretical convergence ($L_{D, \mathrm{opt,real}}$, $L_{D, \mathrm{opt,fake}}$, and $L_{G, \mathrm{opt}}$). 
Moreover, comparing the one-sided flooding experiments (Figure \ref{fig:synthetic-loss-transition} (c) and (d)), while overfitting can be observed in one-sided (real) flooding, one-sided (fake) flooding stably converged. 
This suggests that overfitting to generated data is occurring in synthetic dataset experiments, and preventing this through flooding leads to stabilization.
We also consider that Figure \ref{fig:synthetic-loss-transition} supports the concept of Figure \ref{flooding_effect_concept} that flooding prevent rapid decline in the discriminator's loss, while Figure \ref{fig:synthetic-loss-transition} (a) does not show such rapid decline in Figure \ref{flooding_effect_concept} (a).
We believe that the drop in the batch loss shown in Figure \ref{fig:synthetic-loss-transition} (a) become smaller because overfitting at the instance level, as shown in Figure \ref{fig:synthetic-loss-transition} (a), occasionally occurs and the losses of overfitted instances and non-overfitted instances were averaged during the calculation of the batch loss.

Another interesting information in Figure \ref{fig:synthetic-loss-transition} is that, compared to w/o flooding, the discriminator's loss around 10,000 iterations is raised according to the flood level Medium, which is not too-low flood level.
The findings suggest that the discriminator does not classify perfectly according to the difference of the real and generated data distribution, which does not correspond to the proof procedure~\cite{NIPS2014_5ca3e9b1}, demanding an optimal discriminator that minimizes the loss for a fixed generator. 
On the other hand, according to the result of Arojovsky et al.~\cite{arjovsky2017towards} that such an optimal discriminator can cause instability, preventing the optimal discriminator by using flooding could be regarded as the advantage.
Note that Figure \ref{fig:synthetic-loss-transition} also shows the existence of the loss dynamics in GANs that do not follow the proof of procedure~\cite{NIPS2014_5ca3e9b1}, the details are future work.

Moreover, we investigated the model's gradient norms, which stability affects the training stability.

Figure \ref{fig:synthetic-gradient-transition} shows a graph of the transition of gradient norms during the experiment in Figure \ref{fig:synthetic-loss-transition}. 
Without flooding for $L_{D, \mathrm{fake}}$ ((a) and (c)), the gradient is large from the beginning and becomes even larger as training progresses.
It indicates instability in training as it can be confirmed even when the loss appears to be stable in 20,000~40,000 iterations in Figure \ref{fig:synthetic-loss-transition}. 
On the other hand, with flooding for $L_{D, \mathrm{fake}}$ ((b) and (d)), the gradient peak in the early stages of the training and then gradually decreases.
Furthermore, the gradient norms at the peak are also relatively small.
It suggests that flooding also affects suppressing the gradient and stabilize the training.

\section{Generated images}

We provide the generated images that we could not fully display in Figure \ref{fig:flooding_bce_images} and \ref{fig:flooding_stargan_images}.
Figure \ref{fig:celebA128_12x4} shows the generated images with DCGAN, CelebA (128$\times$128), and the BCE loss.
Moreover, Figure \ref{fig:flooding_stargan_images_supp_latent} and \ref{fig:flooding_stargan_images_supp_reference} show the generated images with StarGAN V2 and CelebAHQ.

\begin{figure}[t]
\centering
 \begin{minipage}[b]{0.9\linewidth}
  \centering
  \includegraphics[keepaspectratio, width=0.99\linewidth]{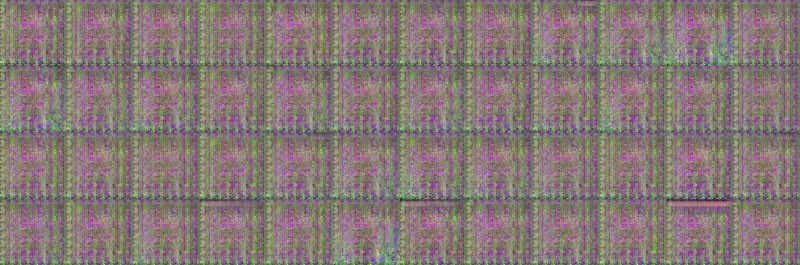}
  \subcaption{w/o flooding (collapse)}
 \end{minipage} \\
 \begin{minipage}[b]{0.9\linewidth}
  \centering
  \includegraphics[keepaspectratio, width=0.99\linewidth]{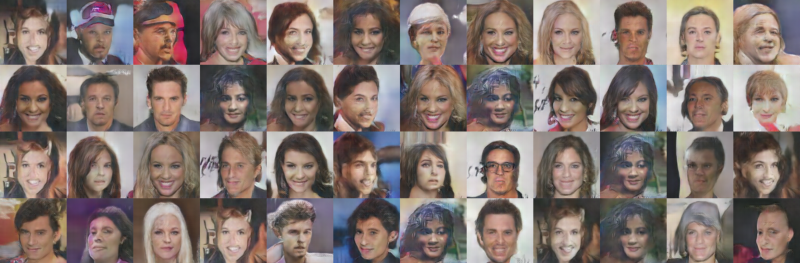}
  \subcaption{w/o flooding}
 \end{minipage} \\
 \begin{minipage}[b]{0.9\linewidth}
  \centering
  \includegraphics[keepaspectratio, width=0.99\linewidth]{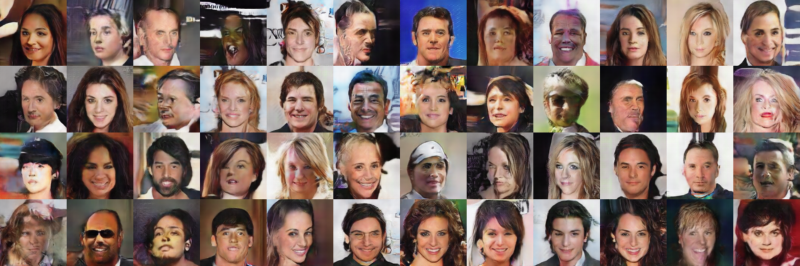}
  \subcaption{With flooding}
 \end{minipage} \\
 \caption{
    Generated images with DCGAN, CelebA (128$\times$128), and the BCE loss that we could not fully display in Figure \ref{fig:flooding_bce_images}.
    Without flooding, in four out of five trials the generated images were collapsed (a).
    In one trial, the generated images did not collapsed, however, causes the graying (b).
    On the other hand, the generated images of four trials with flooding are not collapsed (c).
 }
\label{fig:celebA128_12x4}
\end{figure}

\begin{figure}[t]
\begin{center}
  \includegraphics[keepaspectratio, width=1.0\linewidth]{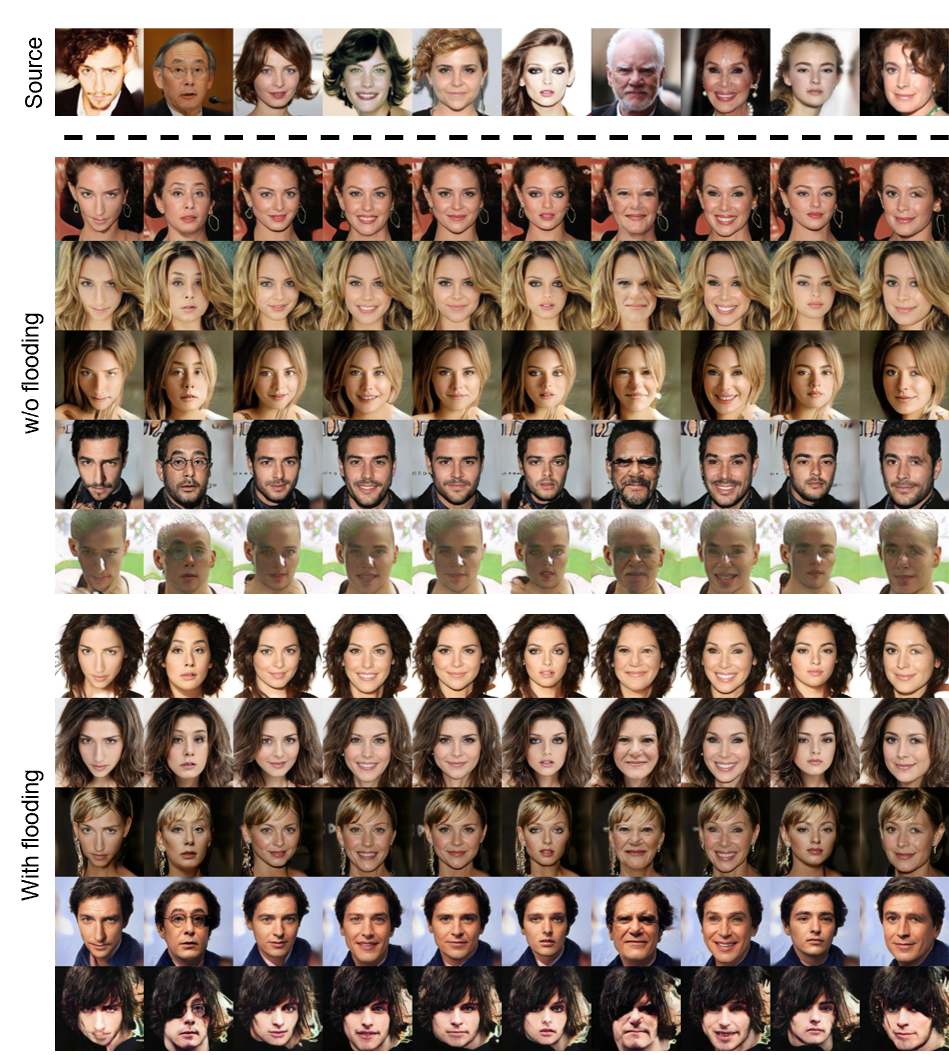}
\end{center}
 \caption{
 Generated images with source images and random latent vectors for style codes.
 We arrange the source images in the top row, and below them, we show the generated images, both with and without flooding.
 }
\label{fig:flooding_stargan_images_supp_latent}
\end{figure}

\begin{figure}[t]
\begin{center}
  \includegraphics[keepaspectratio, width=1.0\linewidth]{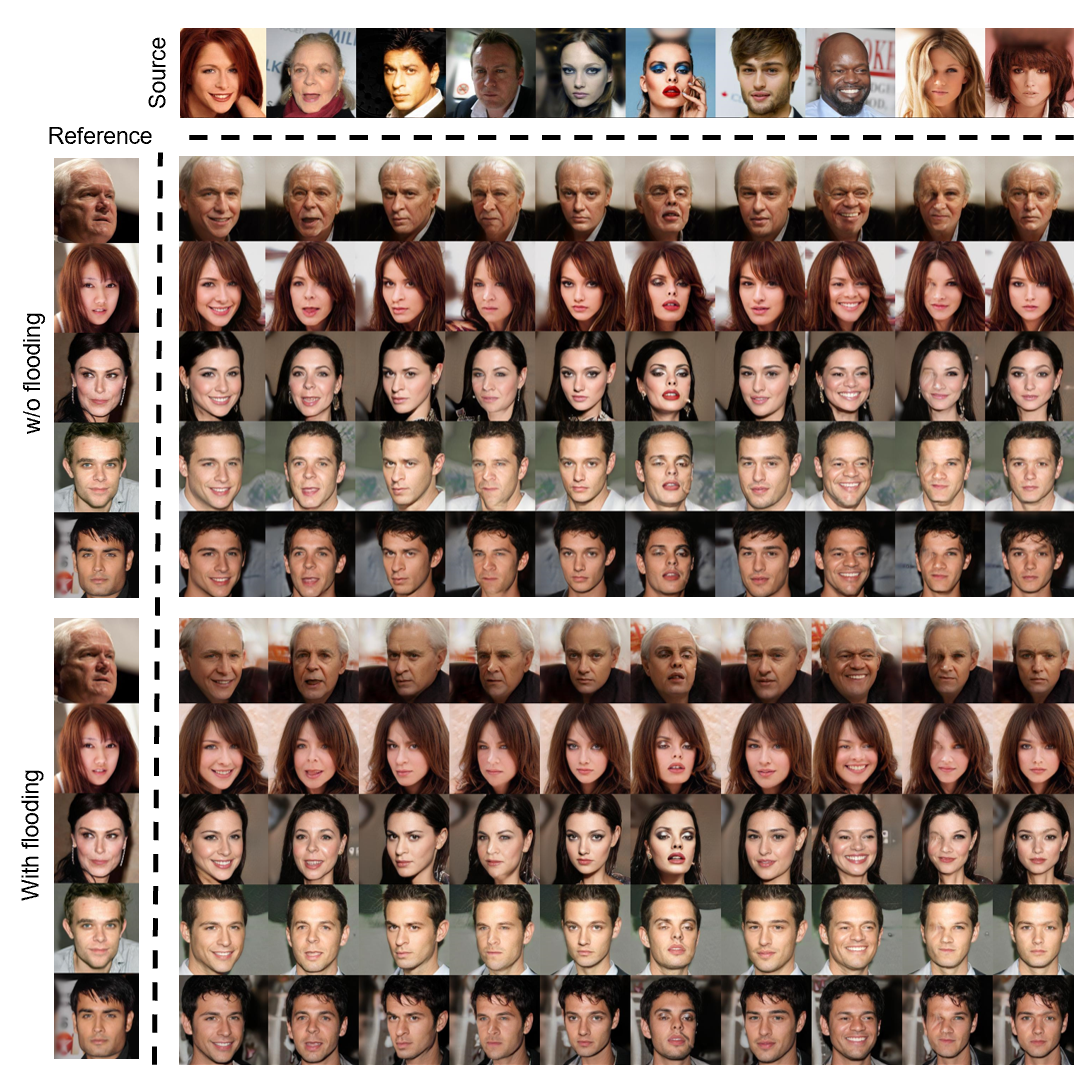}
\end{center}
 \caption{
 Generated images with source images and reference images for style codes.
 We arrange the source images in the top row, and below them, we show the generated images, both with and without flooding.
 We also arrange the reference images in the left column.
 }
\label{fig:flooding_stargan_images_supp_reference}
\end{figure}

\end{document}